\PassOptionsToPackage{table}{xcolor}
\documentclass[11pt, a4paper]{utda} 
\usepackage[numbers,sort&compress]{natbib}

\usepackage{amsmath,amsfonts,bm}

\def\1{\bm{1}}

\def\vm{{\bm{m}}}

\def\vu{{\bm{u}}}
\def\vv{{\bm{v}}}
\def\vw{{\bm{w}}}

\def\mC{{\bm{C}}}

\def\mW{{\bm{W}}}

\DeclareMathAlphabet{\mathsfit}{\encodingdefault}{\sfdefault}{m}{sl}
\SetMathAlphabet{\mathsfit}{bold}{\encodingdefault}{\sfdefault}{bx}{n}

\newcommand{\E}{\mathbb{E}}
\newcommand{\Ls}{\mathcal{L}}
\newcommand{\R}{\mathbb{R}}

\newcommand{\norm}[1]{\Vert #1 \Vert}
\newcommand{\abs}[1]{\left\vert #1 \right\vert}

\DeclareMathOperator*{\argmax}{arg\,max}

\usepackage{booktabs}
\usepackage{url}
\usepackage{fontawesome5}          
\usepackage{xurl}                 
\usepackage{amsmath}
\usepackage{amsthm}
\usepackage{mathtools}
\usepackage{graphicx}
\usepackage{hyperref}
\usepackage{algorithm, algpseudocode}
\usepackage{wrapfig}
\usepackage{framed}
\usepackage{titletoc}             

\usepackage{multirow}             
\usepackage[table]{xcolor}        
\definecolor{lightgray}{gray}{0.9}
\definecolor{linkblue}{RGB}{0,0,250}
\definecolor{uclablue}{RGB}{39,116,174}
\definecolor{darkblue}{RGB}{0,0,200}
\hypersetup{colorlinks=true, linkcolor=linkblue, citecolor=uclablue, urlcolor=darkblue,
  pdftitle={Rectified Language Flow}, pdfauthor={Baiyu Su, Runlong Liao, Lizhang Chen, Qiang Liu}}

\makeatletter
\renewcommand{\paragraph}{%
  \@startsection{paragraph}{4}%
  {\z@}{0ex \@plus 0ex \@minus 0ex}{-1em}%
  {\normalfont\normalsize\bfseries}%
}
\makeatother

\usepackage{xspace}
\newcommand{\dflow}{ReLF\xspace}

\newtheorem{theorem}{Theorem}[section]

\newtheorem{lemma}[theorem]{Lemma}
\newtheorem{proposition}[theorem]{Proposition}   
\usepackage[createShortEnv]{proof-at-the-end}

\definecolor{findingBlue}{HTML}{0153D6}
\colorlet{findingFill}{findingBlue!8}
\newcounter{finding}
\makeatletter
\newcommand{\experimentfinding}[1]{%
  \par\vspace{2pt}\noindent
  \refstepcounter{finding}%
  \begingroup\setlength{\fboxsep}{0pt}%
  \colorbox{findingFill}{%
    \hbox to\linewidth{%
      {\color{findingBlue}\vrule width 2.5pt}%
      \hskip 8pt
      \parbox{\dimexpr\linewidth-2.5pt-16pt\relax}{%
        \rule{0pt}{\dimexpr\ht\strutbox+3pt\relax}
        {\color{findingBlue}\bfseries Observation \thefinding.}\hspace{1em}#1%
        \rule[-\dimexpr\dp\strutbox+3pt\relax]{0pt}{0pt}}
      \hskip 8pt}}%
  \endgroup
  \par\vspace{4pt}%
  \nobreak\@afterheading
}
\makeatother

\newcommand{\relflogo}{\raisebox{-0.26em}{\includegraphics[height=1.3em]{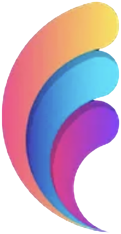}}}
\title{\texorpdfstring{\relflogo\hspace{0.25em}}{}Rectified Language Flow}
\renewcommand{\today}{} 

\author{
\textbf{Baiyu Su}\textsuperscript{*},
\textbf{Runlong Liao}\textsuperscript{*},
\textbf{Lizhang Chen}\textsuperscript{\dag},
\textbf{Qiang Liu}\\[0.5ex]
The University of Texas at Austin
}

\correspondingauthor{\textsuperscript{*}Equal contribution.\quad \textsuperscript{\dag}Correspondence:\ \href{mailto:lzchen@utexas.edu}{\faEnvelope[regular]\ lzchen@utexas.edu}\quad{\color{black!25}\textbar}\quad\href{https://rectifiedflow.github.io/ReLF/}{\faGlobe\ Website}}

\setlist[itemize]{
    leftmargin=*,
    itemsep=0.3em,
    parsep=0pt,
    topsep=0pt,
    partopsep=0pt
}

\begin{document}

 \begin{abstract}
Continuous-state flow and diffusion models offer a promising approach to fast, parallel language generation, but matching the generation quality of autoregressive models remains challenging. We present Rectified Language Flow (\dflow), a language model that generates text by transporting noise to token embeddings using rectified flow. \dflow combines several complementary design choices: fixed token embeddings constructed with Hellinger–SVD, self-conditioning guidance trained with cross-entropy, and regularized reflow for few-step generation. We further introduce singular logits, which combine an analytically computed term that drives commitment to a valid token embedding as noise vanishes with a bounded neural residual that learns contextual token preferences. With 64 function evaluations on OpenWebText, \dflow achieves lower generative perplexity than autoregressive baselines at matched diversity across all three model sizes tested, while also outperforming prior continuous flow methods in this setting. 
After two reflow rounds, \dflow still outperforms autoregressive baselines with 16 function evaluations while generating 1,024-token sequences 15--19$\times$ faster at batch size 1.
\end{abstract}

\maketitle

\begin{figure}[!ht]
    \centering
    \setlength{\abovecaptionskip}{2pt}
    \begin{minipage}[b]{0.49\textwidth}
        \centering
        \includegraphics[width=\linewidth]{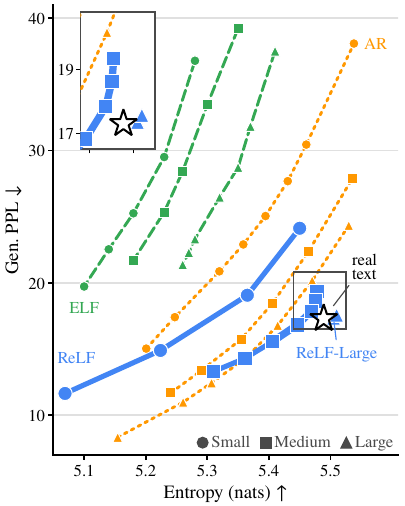}
        \par\small\textbf{(a) 64 NFE}
    \end{minipage}\hfill
    \begin{minipage}[b]{0.49\textwidth}
        \centering
        \includegraphics[width=\linewidth]{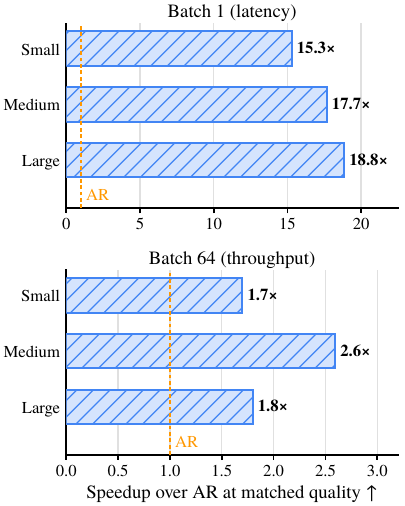}
        \par\small\textbf{(b) Wall-clock speedup}
    \end{minipage}
\caption{\small
\textbf{(a)} Quality–diversity Pareto frontiers for \dflow and ELF~\citep{elf2026} at 64 NFE, with AR baselines; curves are traced by varying guidance scale for flows and temperature for AR.
\dflow improves both metrics over same-size ELF models. \dflow-Large reaches Generative PPL $17.3$ at entropy $5.51$, close to real text ($17.3$, $5.49$).
\textbf{(b)} Wall-clock speedup over same-size AR models for 1,024-token generation on one GPU. With 16 NFE after two reflow rounds, \dflow achieves lower Generative PPL than these AR models at similar entropy.
}    
    \label{fig:system_comparison}
\end{figure}



\section{Introduction}
\label{sec:introduction}

Parallel text generation can reduce the sequential cost of autoregressive (AR) decoding. Flow and diffusion models take two approaches: \emph{discrete-state} models update token identities, whereas \emph{continuous-state} models evolve real-valued representations decoded into tokens. 
Discrete approaches have attracted substantial attention through large-scale diffusion language modeling~\citep{sahoo2024simple, nie2025large, song2025seed}. Building on early continuous text diffusion~\citep{li2022diffusion, gulrajani2023likelihood}, continuous-state approaches are also gaining momentum, with recent developments in embedding-space modeling and few-step generation~\citep{lee2026flow, elf2026}.

Continuous formulations are attractive for two reasons. First, they are flexible: all positions evolve jointly before committing to tokens, so a model transforming noise across the whole sequence can, in principle, capture inter-token dependencies in a single evaluation, the basis for one-step generation~\citep{lee2026flow}. Second, they inherit image, audio, and video generation tools, namely ODE/SDE samplers, guidance, and trajectory straightening, for controlling generation and accelerating sampling~\citep{liu2022flow, lipman2022flow}.

However, turning these advantages into effective language models raises three design challenges: learning continuous dynamics from categorical supervision; reliably recovering tokens at the flow's end; and designing token embeddings whose geometry supports corruption and denoising. Prior work explored auxiliary objectives~\citep{li2022diffusion, gulrajani2023likelihood}, encoder--decoder networks~\citep{lovelace2023latent}, and vocabulary-sized simplex representations~\citep{han2022ssd, lee2026flow}. We seek a simple continuous-flow formulation that addresses them while generating high-quality, diverse text with few function evaluations.


\paragraph{Results.} 
We present \emph{\textbf{Re}ctified \textbf{L}anguage \textbf{F}low} (\dflow), a rectified-flow language model that generates token sequences along continuous embedding trajectories from noise. \dflow predicts a token posterior at each position and is trained only with cross-entropy. The velocity field is computed from the posterior-weighted mean of the token embeddings.
Our OpenWebText experiments show:
\begin{itemize}

\item \emph{Generation quality at 64 function evaluations.}
    \dflow outperforms same-size AR baselines and Embedded Language Flows (ELF)~\citep{elf2026} across all evaluated model sizes. \dflow-Small achieves a generative perplexity of $24.1$ at an entropy of $5.45$ nats, compared with $27.7$ at $5.43$ for AR-Small. \dflow-Large reaches $17.3$ at $5.51$, matching the generative perplexity of real OpenWebText text ($17.3$ at $5.49$).

    \item \emph{Faster generation through reflow.}
    We further apply reflow~\citep{liu2022flow} to reduce the number of function evaluations. After two rounds, \dflow still outperforms same-size AR models with just 16 function evaluations. It then generates 1,024-token sequences 15--19$\times$ faster than these AR baselines at batch size 1 (Figure~\ref{fig:system_comparison}(b)).

    \item \emph{Generation quality at 8 function evaluations.}
    With only 8 function evaluations, \dflow-Medium achieves a generative perplexity of $42.7$, compared with $66.4$ for ELF+PD, a progressively distilled ELF model, both at entropy $5.42$.
\end{itemize}

\paragraph{Methods.}
\dflow learns continuous dynamics from categorical supervision by training token posteriors with cross-entropy (Section~\ref{sec:rectified_flow}). Four components then address reliable token recovery, embedding geometry, and high-quality generation with few function evaluations:

\begin{itemize}
    \item \textbf{Singular logits.}
    As noise vanishes, token posteriors approach point masses, requiring unbounded logit differences. Our \emph{singular logits} capture this divergence analytically through a reference posterior, so the network learns only a bounded residual of contextual token preferences. This guarantees that the modeled posterior converges to a one-hot distribution and the trajectory to a token embedding (Section~\ref{sec:singular_logits}).

    \item \textbf{Embedding geometry for language flows.}
    Token embeddings define the geometry of corruption and denoising. Our fixed, frequency-weighted Hellinger--SVD embeddings place tokens with similar contexts nearby. Their dimension is a free hyperparameter, which we find important for generation quality (Section~\ref{sec:embedding_training}).

    \item \textbf{Logit-space self-conditioning guidance.}
    Self-conditioning guidance extrapolates between predictions with and without the previous clean estimate, which prior work does in continuous state space~\citep{chen2022analog, elf2026}. We instead extrapolate in logit space, so the guided prediction remains a token distribution, learnable through cross-entropy alone with one function evaluation per step (Section~\ref{sec:self_conditioning}).

    \item \textbf{Regularized reflow.}
    \emph{Regularized reflow} retrains the model on its own noise--token couplings and interpolates its weights back toward the base model, straightening the flow for few-step generation while controlling drift (Section~\ref{sec:reflow}).
\end{itemize}

\section{Continuous Flow for Discrete Generation} \label{sec:rectified_flow}

\paragraph{Setup.}
Let $\mathcal V=\{1,\ldots,V\}$ be the vocabulary and $\vw=[w^{(1)},\ldots,w^{(L)}]\in\mathcal V^L$ a token sequence drawn from the data distribution $\mathcal D$. An embedding map $E:\mathcal V\to\R^d$ places each token in the \emph{flow embedding space} $\R^d$, where generation takes place; we call $d$ the \emph{flow embedding dimension}. Applied positionwise, $E$ maps $\vw$ to the clean state $X_1=E(\vw)\in\R^{d\times L}$, whose distribution we denote by $q_1$. Noisy states in $\R^{d\times L}$ are written $X_t$ along training interpolations and $Z_t$ along sampling trajectories.
To lighten notation, we suppress position indices: a token distribution such as $p_\theta(w\mid Z_t,t)$ refers to a single position but conditions on the entire noisy sequence, whereas positionwise operations on embeddings use only the state at that position. We write $q_t(w\mid Z_t)$ for the token posterior under the data distribution and $r_t(w\mid Z_t)$ for the posterior under the reference distribution introduced in Section~\ref{sec:singular_logits}. Appendix~\ref{app:theory} gives the fully indexed definitions and the theory.

\begin{figure}[t]
    \centering
    \setlength{\abovecaptionskip}{2pt}
    \setlength{\belowcaptionskip}{0pt}
    \includegraphics[width=\textwidth]{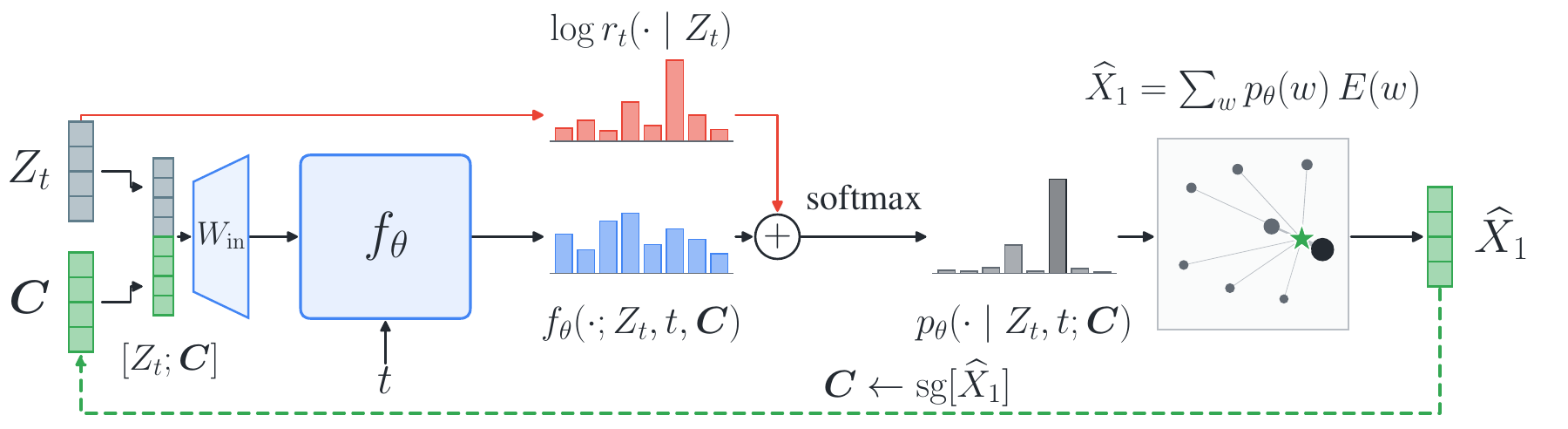}
    \caption{\textbf{\dflow architecture.} From the noisy state $Z_t$ and self-conditioning input $\mC$, the Transformer outputs residual logits $f_\theta$, added to singular logits $\log r_t$. $\widehat{X}_1$ is a probability-weighted mean of token embeddings, fed back via stop-gradient $\operatorname{sg}$ as the next $\mC$ (dashed).}
    \label{fig:architecture}
    \vspace{-8pt}
\end{figure}

\paragraph{Rectified Flow.}
Following the rectified flow (RF) framework of \citet{liu2022flow}, we interpolate between data embeddings $X_1\sim q_1$ and noise $X_0\sim q_0$ as $X_t=tX_1+(1-t)X_0$ for $t\in[0,1]$, and write $q_{t\mid1}(\cdot\mid\vw)$ for the corruption likelihood, the density of $X_t$ given the clean tokens $\vw$.
In standard RF, the model learns a velocity field $\vv_\theta(X,t)$ by minimizing the mean-squared error (MSE) loss:
\begin{equation*}
    \Ls(\theta)=\E_{X_1\sim q_1,\,X_0\sim q_0,\,t\sim\mathcal T}\!\left[\norm{X_1-X_0-\vv_\theta(X_t,t)}^2\right],
\end{equation*}
where $X_1$, $X_0$, and $t$ are independent and $\mathcal T$ is a training-time distribution on $[0,1]$.
The associated flow ODE, initialized at $Z_0\sim q_0$, evolves as
\begin{equation}\label{eq:rf_ode}
    \frac{\mathrm dZ_t}{\mathrm dt}=\vv_\theta(Z_t,t),
    \qquad
    \vv_\theta(Z_t,t)\approx\vv^*(Z_t,t)=\E[X_1-X_0\mid X_t=Z_t],
\end{equation}
so that $\mathrm{Law}(Z_t)\approx\mathrm{Law}(X_t)$ for $t\in[0,1]$.

\paragraph{Cross-Entropy Training.}
MSE is the default RF objective, but cross-entropy (CE) is more natural for language. RF admits CE training because its velocity field can be expressed through token posteriors. The optimal velocity is determined by the conditional mean of the clean embedding, which is a posterior-weighted average of token embeddings (Appendix~\ref{app:velocity_marginals}):
\begin{equation}\label{eq:actipated_embedding}
    \vv^*(X_t,t)=\tfrac{1}{1-t}\bigl(\E[X_1\mid X_t]-X_t\bigr),
    \qquad
    \E[X_1\mid X_t]=\sum_{w\in\mathcal V}E(w)\,q_t(w\mid X_t).
\end{equation}
Learning $\vv^*$ therefore reduces to learning the positionwise posterior marginals $q_t(w\mid X_t)$, a classification problem that CE addresses directly~\citep{DBLP:journals/corr/abs-2211-15089, han2022ssd, eijkelboom2024variational}. A bidirectional diffusion Transformer predicts these marginals jointly from the full noisy sequence; we denote its predictions by $p_\theta(w\mid X_t,t)$ and train it with
\begin{equation}\label{eq:token_ce}
    \Ls(\theta)=\E_{\vw\sim\mathcal D,\,X_0\sim q_0,\,t\sim\mathcal T}\!\left[-\log p_\theta(w\mid X_t,t)\right].
\end{equation}
The expectation also averages over token positions, with $w$ denoting the clean token at each position. The model's clean-state estimate $\widehat{X}_1(X_t,t)=\sum_{w\in\mathcal V}E(w)p_\theta(w\mid X_t,t)$ defines the velocity $\vv_\theta=\tfrac{1}{1-t}(\widehat X_1-X_t)$. CE is minimized at $p_\theta=q_t$, where $\vv_\theta=\vv^*$, so the CE and MSE objectives share the same optimal velocity.

\paragraph{Sampling.}
\dflow samples like a standard continuous flow except at the final step. Starting from $Z_{t_0}\sim q_0$, we integrate the flow ODE with Euler steps on a grid $0=t_0<t_1<\cdots<t_{N-1}<1$:
\begin{equation}\label{eq:euler_decoding}
    Z_{t_{i+1}}=Z_{t_i}+(t_{i+1}-t_i)\tfrac{1}{1-t_i}\bigl(\widehat{X}_1(Z_{t_i},t_i)-Z_{t_i}\bigr).
\end{equation}
At the final step, we decode each position as $\widehat w=\argmax_{w\in\mathcal V}p_\theta(w\mid Z_{t_{N-1}},t_{N-1})$. We also support SDE sampling with the score-corrected formulation of \citet{ma2024sit}, whose noise scale $\sigma\geq0$ controls stochasticity; $\sigma=0$ recovers the flow ODE (Appendix~\ref{app:sde_discretization}).

\section{Commitment, Geometry, Guidance, and Straightening}
\label{sec:improved_techniques}

Section~\ref{sec:rectified_flow} fixes the training objective and sampler but leaves open how tokens commit as noise vanishes and in which embedding space the flow evolves. This section addresses both: Section~\ref{sec:singular_logits} introduces singular logits, which guarantee convergence to token embeddings, and Section~\ref{sec:embedding_training} constructs compact, context-informed flow embeddings. We further adapt two widely used flow techniques: Section~\ref{sec:self_conditioning} moves self-conditioning guidance to logit space, where it can be trained with CE, and Section~\ref{sec:reflow} stabilizes reflow for few-step generation.

\subsection{Singular Logits Enable Endpoint Convergence}
\label{sec:singular_logits}
\paragraph{Motivation.}
As noise vanishes, the posterior converges to a one-hot distribution. Under a fixed-temperature softmax, this requires diverging logit differences, which a network with bounded inputs and parameters cannot represent. The predicted posterior thus keeps probability on competing tokens even at the endpoint. Since the trajectory converges to the clean-state estimate $\widehat X_1$, it ends at a mixture of token embeddings rather than at a single token embedding (Figure~\ref{fig:side_by_side}(a,c)). Prior continuous models instead decode final states by rounding or a separate decoder~\citep{li2022diffusion,elf2026}. We observe that the exact posterior itself points to a solution: its logits decompose into a \emph{singular} component that diverges as noise vanishes and a \emph{bounded} component carrying contextual token preferences. The singular component has a closed form for a simple reference distribution $r_1$, leaving the network to learn a residual that is guaranteed to stay finite.

\begin{figure}[t]
     \centering
    \setlength{\abovecaptionskip}{0pt}
    \includegraphics[width=\textwidth]{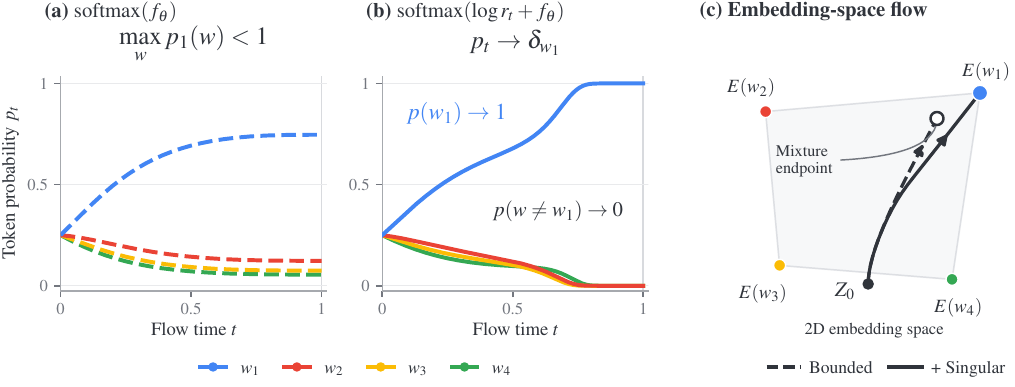}
    \vspace{-2pt}
    \caption{\textbf{Singular logits connect commitment to embedding convergence.} A four-token example in which both models share the bounded residual $f_\theta$ and the initial state $Z_0$, with $p_t=p_\theta(\cdot\mid Z_t,t)$ along each trajectory. \textbf{(a)} With bounded logits, $p_t$ retains mass on competing tokens. \textbf{(b)} Adding $\log r_t$ drives $p_t$ to a one-hot distribution. \textbf{(c)} Correspondingly, the flow ends at an embedding mixture (dashed) or at $E(w_1)$ (solid).}
    \label{fig:side_by_side}
    \vspace{-1.0em}
\end{figure}

\paragraph{Singular-Logit Correction.}
We compute the log reference posterior $\log r_t(w\mid Z_t)$, which we call the \emph{singular logits}, analytically, and let the network parameterize only a residual correction $f_\theta$ toward the data posterior:
\begin{equation}
    p_\theta(\cdot\mid Z_t,t)\coloneqq\operatorname{softmax}\!\left(\log r_t(\cdot\mid Z_t)+f_\theta(\cdot;Z_t,t)\right).
    \label{eq:singular_logit_correction}
\end{equation}
The exact residual is bounded whenever the reference stays within a constant factor of the data distribution.

\begin{theorem}
Suppose $\kappa^{-1}\leq q_1(\vw)/r_1(\vw)\leq\kappa$ for every sequence $\vw$ and some $1\leq\kappa<\infty$. Then the exact residual $f_*(w;Z_t,t)=\log q_t(w\mid Z_t)-\log r_t(w\mid Z_t)$, which $f_\theta$ approximates, satisfies $\left|f_*(w;Z_t,t)\right|\leq 2\log\kappa$ for all $t<1$.
\end{theorem}

With a full-support reference such as the uniform distribution, the ratio condition is equivalent to $q_1(\vw)>0$ for every sequence, since $L$ and $\mathcal V$ are finite. This is the same condition AR models need for exact representation with finite softmax logits. The resulting residual bound is uniform in $t$, so the network need not represent the divergence caused by vanishing noise. Singular logits require no learning and are cheap to compute: with a factorized reference, the reference posterior reduces to a vocabulary sum at each position.

\begin{lemma}
Suppose the reference distribution factorizes as $r_1(\vw)=\prod_{\ell=1}^L r_1^{(\ell)}(w^{(\ell)})$, where $r_1^{(\ell)}$ is the reference token distribution at position $\ell$, and the noise is independent across positions, so that $q_{t\mid1}$ factorizes as well. Then the singular logits can also be computed positionwise as
\[
\log r_t(w\mid Z_t)
=\log\frac{r_1(w)q_{t\mid1}(Z_t\mid w)}
{\sum_{u\in\mathcal V}r_1(u)q_{t\mid1}(Z_t\mid u)}.
\]
Here $r_1(w)$ and $q_{t\mid1}(Z_t\mid w)$ are the factors at the corresponding position.
\end{lemma}

For standard Gaussian noise, $q_{t\mid1}(Z_t\mid w)\propto\exp\!\bigl(-\norm{Z_t-tE(w)}^2/(2(1-t)^2)\bigr)$ with a token-independent constant, so
\begin{equation}
    \log r_t(\cdot\mid Z_t)=\log\operatorname{softmax}\!\left(\log r_1(\cdot)-\frac{\norm{Z_t-tE(\cdot)}^2}{2(1-t)^2}\right),
    \label{eq:singular_logits_gaussian}
\end{equation}
whose distance term diverges as $t\to1$.
Factorization applies only to the reference; the network still conditions on the full sequence. Appendix~\ref{app:singular_logits} gives the proofs.

With singular logits, the modeled posterior converges to a one-hot distribution whenever the reference posterior does, for any network with bounded outputs; this is token commitment in \emph{probability space}. The following proposition shows that this implies commitment in \emph{flow embedding space}: the ODE trajectory converges to the selected token embedding.

\begin{proposition}
Consider the ODE in Eq.~\eqref{eq:rf_ode} with the velocity of Eq.~\eqref{eq:actipated_embedding}, computed from the modeled probabilities $p_\theta$ in place of $q_t$. If each position's $p_\theta(\cdot\mid Z_t,t)$ converges to a one-hot vector selecting a token $w$ as $t\uparrow1$, then its corresponding state satisfies $Z_t\to E(w)$.
\end{proposition}

Thus, under exact continuous-time dynamics, every trajectory reaches a token embedding regardless of how well the network is trained (Appendix~\ref{app:endpoint_convergence}; Figure~\ref{fig:side_by_side}(b,c)). Singular logits ensure commitment, while the learned residual decides which tokens are chosen and hence the quality of the generated text. On OpenWebText, singular logits improve the quality--diversity trade-off (Section~\ref{sec:ablations}). They also add little computation: because the flow embedding dimension $d=32$ (Section~\ref{sec:embedding_training}) is much smaller than the Transformer hidden width, their share of the total FLOPs per inference step ranges from approximately $1.1\%$ for \dflow-Small to $0.2\%$ for \dflow-Large.

\subsection{Construction of Flow Embeddings} 
\label{sec:embedding_training}
The embedding matrix $E$ determines which tokens remain ambiguous under noise. Any embedding with distinct rows recovers the data distribution under exact posteriors and integration, but its geometry controls how difficult these posteriors are to approximate. Nearby embeddings should represent tokens with similar contexts, allowing noisy states to preserve information about neighboring positions before token commitment.

\paragraph{Hellinger--SVD.}
Following \citet{lebret2014hellinger}, we estimate each token $a$'s smoothed distribution $Q_a(c)$ over preceding and following contexts $c$. Its square-root representation $H_{ac}=\sqrt{Q_a(c)}$ makes Euclidean distance proportional to Hellinger distance:
\begin{equation}
    \norm{H_{a,:}-H_{b,:}}^2
    =\sum_c\left(\sqrt{Q_a(c)}-\sqrt{Q_b(c)}\right)^2
    =2d_{\mathrm{Hel}}^2(Q_a,Q_b).
    \label{eq:embedding_hellinger_geometry}
\end{equation}
For compact embeddings, let $\bar H$ be the frequency-weighted mean and $P_d$ contain the leading $d$ right singular vectors of the frequency-weighted, centered matrix. We project onto these directions, $\widetilde E(a)=(H_{a,:}-\bar H)P_d$, which retains the context geometry captured by the leading components (Appendix~\ref{app:hellinger_svd}).

After perturbing duplicate rows, we center and rescale the embeddings under the unigram distribution $\pi$ to satisfy $\sum_w\pi(w)E(w)=0$ and $\sum_w\pi(w)\norm{E(w)}^2=d$, fixing average signal power relative to unit Gaussian noise. To allocate more training to noise levels where token identification improves fastest, we sample training times uniformly along the joint arc length of two normalized token-identification curves, which measure posterior log-probability gains over the prior with and without local context. The embeddings and this schedule remain fixed during training (Appendix~\ref{app:embedding_details}).

\paragraph{Embedding Dimension.}
In prior continuous models, the flow dimension is tied to another component: to the vocabulary size in the one-hot and simplex states of FMLM~\citep{lee2026flow} and SSD-LM~\citep{han2022ssd}, to the pretrained T5 encoder in ELF~\citep{elf2026}, and to the model width in LangFlow~\citep{chen2026langflow}. In \dflow, the network projects each $d$-dimensional state to its hidden width and predicts $V$ token logits, so $d$ is decoupled from all three and becomes a free hyperparameter. We use $d=32$ for all model sizes, consistent with compact representations in image diffusion~\citep{zheng2026rae,li2026jit}; Section~\ref{sec:ablations} compares dimensions 16 to 256.

\subsection{Cross-Entropy Training of Self-Conditioning}
\label{sec:self_conditioning}
We use \emph{self-conditioning}~\citep{chen2022analog, chen2026langflow}: the network concatenates its noisy input with a clean-state estimate $\mC$ (Figure~\ref{fig:architecture}). At inference, $\mC_0=\mathbf 0$ and $\mC_{i+1}=\widehat{X}_{1,i}$, reusing the previous prediction without additional function evaluations.

\paragraph{Guidance in Logit Space.}
A self-conditioned model produces cold logits $l(\mathbf 0)$ without the self-conditioning input and warm logits $l(\mC)$ with it. Prior self-conditioning guidance~\citep{chen2022analog, elf2026} operates in the continuous data space; ELF, for example, extrapolates its MSE-trained velocity as $v^{\mathrm g}=v(\mathbf 0)+\omega\bigl(v(\mC)-v(\mathbf 0)\bigr)$. When the clean-state estimate is a posterior mean, $\widehat X_1=\sum_{w}E(w)\,p(w)$, the velocity is affine in the token probabilities, so flow-space extrapolation is equivalent to extrapolating the probabilities linearly, $\lambda^{\mathrm g}=p(\mathbf 0)+\omega\bigl(p(\mC)-p(\mathbf 0)\bigr)$, whose entries can turn negative. We instead extrapolate between the cold and warm predictions in \emph{logit space},
\begin{equation}\label{eq:sc_logit_guidance}
    l^{\mathrm g}=l(\mathbf 0)+\omega\bigl(l(\mC)-l(\mathbf 0)\bigr),
\end{equation}
with guidance scale $\omega$, and compute $\widehat X_1=\sum_{w\in\mathcal V}E(w)\operatorname{softmax}(l^{\mathrm g})_w$, which determines the velocity and is carried as the next $\mC$ (Figure~\ref{fig:sc_guidance}, right). Applied explicitly, Eq.~\eqref{eq:sc_logit_guidance} requires a cold and a warm evaluation at every step.

\begin{figure}[t]
    \centering
    \setlength{\abovecaptionskip}{0pt}
    \includegraphics[width=\textwidth]{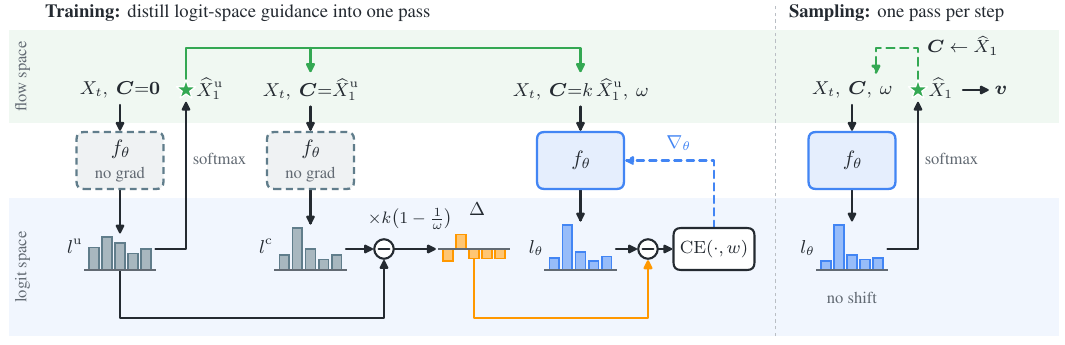}
    \vspace{-2pt}
    \caption{\textbf{Guidance on logits.} \textbf{Training (left):} two no-gradient passes give cold logits $l^{\mathrm u}$ ($\mC=\mathbf 0$) and warm logits $l^{\mathrm c}$ (conditioned on $\widehat X_1^{\mathrm u}$); their difference, scaled by $k(1-1/\omega)$, is the shift $\Delta$ subtracted from the network's logits before CE at the clean token $w$. \textbf{Sampling (right):} one evaluation at the chosen $\omega$ uses the unshifted logits, whose clean-state estimate $\widehat X_1$ gives the velocity and the next $\mC$. Shading separates flow space (top) from logit space (bottom).}
    \label{fig:sc_guidance}
    \vspace{-1.0em}
\end{figure}

\paragraph{Training with Cross-Entropy.}
We distill the guided logits into the network, which takes the guidance scale $\omega$ as an input, so that each sampling step requires only one network evaluation. Because the guidance acts on logits, the guided prediction remains a token distribution, so both denoising and guidance can be learned through the CE objective alone. Let $l_\theta(X_t,t;\mC,\omega)$ denote the model's token logits. Two no-gradient passes produce cold logits $l^{\mathrm u}$ with $\mC=\mathbf 0$ and $\widehat X_1^{\mathrm u}=\sum_{w\in\mathcal V}E(w)\operatorname{softmax}(l^{\mathrm u})_w$, followed by warm logits $l^{\mathrm c}$ conditioned on $\widehat X_1^{\mathrm u}$ (Figure~\ref{fig:sc_guidance}, left). For a per-sequence mask $k\sim\mathrm{Bernoulli}(1/2)$, we set $\mC=k\,\widehat X_1^{\mathrm u}$ and train with
\begin{equation}
    \Ls(\theta)=\E\!\left[-\log\operatorname{softmax}\bigl(l_\theta(X_t,t;\mC,\omega)-\Delta\bigr)_w\right],
    \qquad
    \Delta=k\bigl(1-\tfrac{1}{\omega}\bigr)\operatorname{sg}\bigl[l^{\mathrm c}-l^{\mathrm u}\bigr],
    \label{eq:sc_training_guidance}
\end{equation}
where $w$ is the clean target token and $\operatorname{sg}$ denotes stop-gradient. With $k=0$, the loss is ordinary CE; with $k=1$, subtracting the detached correction trains the unshifted output to incorporate guidance, and at a stationary point it reproduces Eq.~\eqref{eq:sc_logit_guidance}. At $\omega=1$, the correction vanishes, recovering ordinary self-conditioned training. During sampling, we use the unshifted logits at the chosen $\omega$, with one evaluation per step. Appendix~\ref{app:sc_training_guidance} derives the stationary point and gives the training details.

\subsection{Regularized Reflow}
\label{sec:reflow}
Reflow~\citep{liu2022flow} straightens the flow by retraining it on noise--token couplings generated by the model itself, so that sampling needs fewer, larger steps. Let $\mathcal S_\theta(Z_0)$ denote the tokens obtained by integrating the flow ODE of model $\theta$ from $Z_0\sim q_0$ and decoding the final tokens. Reflow fine-tunes a model, initialized at $\theta$, with the CE objective of Eq.~\eqref{eq:token_ce} on these couplings, replacing independent noise with the coupled noise:
\begin{equation}\label{eq:reflow}
    \mathrm{Reflow}(\theta)=\arg \min_{\theta'}\ \E_{Z_0\sim q_0,\,t\sim\mathcal T}\!\left[-\log p_{\theta'}\bigl(\widehat{\vw}\mid tE(\widehat{\vw})+(1-t)Z_0,\,t\bigr)\right],
    \qquad
    \widehat{\vw}=\mathcal S_\theta(Z_0).
\end{equation}
Standard reflow samples with the retrained model directly and draws the next round's couplings from it. We instead regularize every round toward the base model $\theta_0$, which we call \emph{regularized reflow}: with $\theta_k$ the model after $k$ rounds, each round applies reflow and then interpolates the result back toward the base weights,
\begin{equation}\label{eq:reflow_rounds}
    \hat\theta_{k+1}=\mathrm{Reflow}(\theta_k),
    \qquad
    \theta_{k+1}=\lambda\,\theta_0+(1-\lambda)\,\hat\theta_{k+1}.
\end{equation}
The update $\hat\theta_{k+1}-\theta_k$ straightens the flow but also moves the posterior predictions away from those of $\theta_0$, which was fit to the data rather than to finite self-generated couplings. Interpolating in weight space~\citep{wortsman2022robust} keeps much of the straightening while pulling the predictions back toward $\theta_0$, so $\theta_{k+1}$ outperforms both $\theta_0$ and $\hat\theta_{k+1}$ at few steps. Each round draws couplings from, and initializes at, $\theta_k$, so the couplings straighten across rounds, while anchoring to $\theta_0$ prevents drift. The target remains a token posterior, so reflow requires no change to the loss, the singular logits (Section~\ref{sec:singular_logits}), or the architecture. We use two rounds of regularized reflow (Appendix~\ref{app:reflow_training}).

\section{Experiments}
\label{sec:experiments}

\subsection{Experimental Setup}
\label{sec:setup}

\paragraph{Dataset and Evaluation.}
We evaluate unconditional text generation on OpenWebText (OWT)~\citep{gokaslan2019OpenWeb}, following the benchmark used in prior diffusion and flow language models~\citep{lou2023discrete,sahoo2024simple,lee2026flow,elf2026}. We tokenize the corpus with the GPT-2 tokenizer~\citep{radford2019language} and pack sequences to length $L=1024$. For each \dflow and AR sampling configuration, we generate 1,024 sequences (ELF results use 1,000) and report generative perplexity (Gen.\ PPL) under GPT-2 Large and sample diversity as the mean per-sequence unigram entropy in nats. We evaluate the quality--diversity trade-off by sweeping the guidance strength.

\paragraph{Model.}
Our default model is \dflow-Small, a bidirectional diffusion Transformer~\citep{peebles2023dit} with 138M trainable parameters, 12 layers, hidden width 768, and 12 attention heads. We use the fixed 32-dimensional Hellinger--SVD token embeddings constructed in Section~\ref{sec:embedding_training}. At each position, the noisy embedding is concatenated with a self-conditioning estimate of the clean embedding and projected to the Transformer width. The network predicts residual logits, which are added to the analytic singular logits with a unigram reference distribution. Appendix~\ref{app:experimental_details} gives the configurations of all sizes.

\paragraph{Training and Inference.}
We train \dflow-Small, \dflow-Medium, and \dflow-Large for 80,000 steps each, with a global batch of $1{,}024\times1{,}024$ tokens (1,024 sequences of length 1,024), minimizing the CE loss in Eq.~\eqref{eq:token_ce}. We use Muon~\citep{jordan2024muon,Chen2025MuonOU,Chen2024LionSS} for hidden matrices and AdamW for the input and output projections, normalization gains, and biases, with a base learning rate of $3\times10^{-4}$, 2,500 warmup steps, and a constant learning rate thereafter. Training uses standard Gaussian corruption, an embedding-dependent interpolation-time distribution, and self-conditioning with probability $0.5$. For distilled self-conditioning guidance, each sequence also draws a scale $\omega$ log-uniform on $[1.5,6]$, as in ELF~\citep{elf2026}. We evaluate exponential-moving-average (EMA) checkpoints with decay $0.9999$. We report sampling cost in number of function evaluations (NFE) and wall-clock time; by default, we use the SDE sampler with 64 NFE. At sampling time, we sweep $\omega$; because $\omega$ is a network input, each guided step costs a single forward pass.

\subsection{Unconditional Generation on OpenWebText}
\label{sec:headline_benchmark}

Figure~\ref{fig:system_comparison}(a) compares the quality--diversity trade-offs of \dflow, ELF~\citep{elf2026}, and autoregressive (AR) baselines trained on the same data and swept over sampling temperature. Each method has three sizes, Small, Medium, and Large, corresponding to the B, M, and L configurations of \citet{elf2026} (Table~\ref{tab:model_sizes}). With 64 NFE, \dflow outperforms both ELF and the AR baseline at every model size across the trade-off. At matched entropy, \dflow-Small, \dflow-Medium, and \dflow-Large reach Gen.\ PPL $24.1$, $18.6$, and $17.3$, compared with $27.7$, $22.3$, and $20.2$ for AR-Small, AR-Medium, and AR-Large. Appendix~\ref{app:training_guidance_pareto} gives the evaluation details and the full Pareto frontiers.

\experimentfinding{
To the best of our knowledge, \dflow is the first continuous
flow language model to achieve lower generative perplexity
than autoregressive baselines of the same size on OpenWebText
at matched entropy.
}

\textbf{\dflow-Large is close to real OpenWebText on the two reported metrics.} It reaches Gen.\ PPL $17.3$ at entropy $5.51$, matching the Gen.\ PPL of real OpenWebText text ($17.3$ at $5.49$). Unlike the Small and Medium models and despite the identical training pipeline, \dflow-Large no longer trades Gen.\ PPL against entropy effectively under guidance: across all guidance scales, both metrics stay nearly constant, with Gen.\ PPL between $17.3$ and $18.0$ (Table~\ref{tab:l_full_sweep}), so only two of its eight settings are non-dominated and appear in Figure~\ref{fig:system_comparison}(a). This does not mean that its samples fully match real text; the remaining differences in quality and diversity are small, semantic, and beyond what Gen.\ PPL and unigram entropy resolve (see the samples in Appendix~\ref{app:qualitative_examples}). Unconditional OpenWebText generation thus calls for richer benchmarks and metrics.

\paragraph{Few-Step Generation.}
Figure~\ref{fig:fewstep_ladders} shows how regularized reflow (Section~\ref{sec:reflow}) improves few-step generation. Figure~\ref{fig:rounds_slope} in Appendix~\ref{app:additional_results} gives the complementary comparison across rounds at fixed NFE, with round 0 denoting the base model. Across rounds, Gen.\ PPL decreases at 8, 16, and 32 NFE for every model size while unigram entropy remains similar. Most of the gain arrives within two rounds: at 8 NFE, the first round lowers Gen.\ PPL by 29--36\% (e.g., from $62.5$ to $44.5$ for Large), and the second by a further 8--12\%. After two rounds, 16 NFE yields lower Gen.\ PPL than the base model at 32 NFE for every model size (Figure~\ref{fig:reflow_radar} in Appendix~\ref{app:ablation_details}). At matched entropy, the two-round models also attain lower Gen.\ PPL than ELF+PD at every NFE from 8 to 32, and than AR-Small (1,023 NFE) from 16 NFE onward (Table~\ref{tab:elf_b_fewstep}).

\begin{figure}[t]
    \centering
    \includegraphics[width=\textwidth]{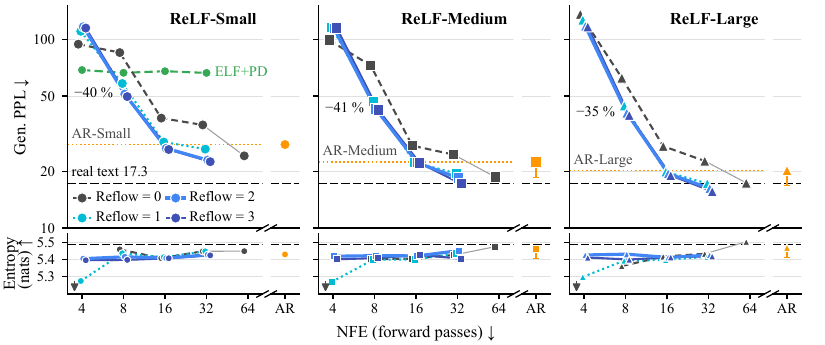}
    \caption{\textbf{Few-step generation across reflow rounds.} Gen.\ PPL (top) and entropy (bottom) versus NFE for Small, Medium, and Large \dflow models after zero to three reflow rounds. At 8 NFE, two reflow rounds reduce Gen.\ PPL by $40\%$, $41\%$, and $35\%$ for Small, Medium, and Large, respectively, relative to the model before reflow. After two rounds, \dflow-Small outperforms ELF+PD~\citep{elf2026} from 8 NFE, and all three sizes match or outperform their AR references from 16 NFE. Dotted lines denote same-size AR models; long-dashed lines denote real OpenWebText. Arrows at 4 NFE mark entropies of the model before reflow that fall below the plotted range.}
    \label{fig:fewstep_ladders}
\end{figure}

\paragraph{Wall-Clock Efficiency.}
Figure~\ref{fig:system_comparison}(b) compares \dflow-Small, \dflow-Medium, and \dflow-Large (138M, 421M, and 784M parameters) with same-size AR models. Because two reflow rounds suffice for each to reach lower Gen.\ PPL than its AR counterpart at similar unigram entropy with 16 NFE, we evaluate \dflow at this budget. We time 1,024-token generation on one GPU, with KV caching for AR and CUDA graphs for both methods, defining speedup as AR time divided by \dflow time. We compare a latency-bound regime, measuring latency at batch size 1, with a compute-bound regime, measuring throughput at batch size 64. At batch size 1, the speedups for Small, Medium, and Large are $15.3\times$, $17.7\times$, and $18.9\times$; at batch size 64, they are $1.7\times$, $2.6\times$, and $1.8\times$. Figure~\ref{fig:latency_throughput} in Appendix~\ref{app:additional_results} reports latency and throughput across batch sizes.

\experimentfinding{Two rounds of regularized reflow reduce Gen.\ PPL at 8 NFE by 35--41\% while unigram entropy remains similar, and 16 NFE then yields lower Gen.\ PPL than the base model at 32 NFE. At 16 NFE, \dflow attains lower Gen.\ PPL than same-size AR models at similar entropy while generating 1,024-token sequences more than $15\times$ faster at batch size 1.}

\subsection{Ablation Studies}
\label{sec:ablations}

Figure~\ref{fig:ablations} summarizes the ablations. Unless otherwise noted, they use \dflow-Small with SDE sampling at 64 NFE and 1,024 samples per point. Panels (a)--(d) plot quality--diversity curves; (c) and (d) sweep the guidance scale $\omega$. Training with distilled self-conditioning guidance was unstable without singular logits in (a) and with one-hot embeddings in (b), and one-hot embeddings do not support SDE sampling. We therefore train all models in (a) and (b) without distilled self-conditioning guidance, sweep momentum guidance~\citep{liao2026momentum} of strength $\alpha$ instead, and also report ODE (Euler) sampling in (b). Appendix~\ref{app:ablation_details} gives further settings.

\begin{figure}[t]
    \centering
    \setlength{\abovecaptionskip}{3pt}
    \setlength{\belowcaptionskip}{0pt}
    \begin{minipage}[t]{0.32\textwidth}
        \centering
        \includegraphics[width=\linewidth]{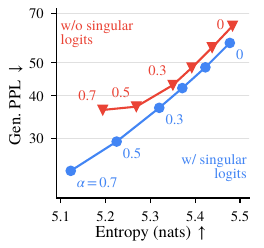}
        \par\small\textbf{(a)}
    \end{minipage}\hfill
    \begin{minipage}[t]{0.32\textwidth}
        \centering
        \includegraphics[width=\linewidth]{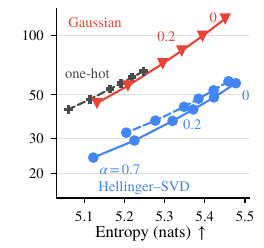}
        \par\small\textbf{(b)}
    \end{minipage}\hfill
    \begin{minipage}[t]{0.32\textwidth}
        \centering
        \includegraphics[width=\linewidth]{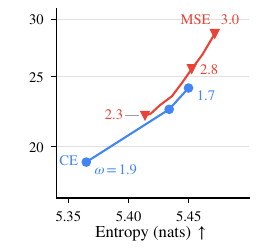}
        \par\small\textbf{(c)}
    \end{minipage}\\[4pt]
    \begin{minipage}[t]{0.32\textwidth}
        \centering
        \includegraphics[width=\linewidth]{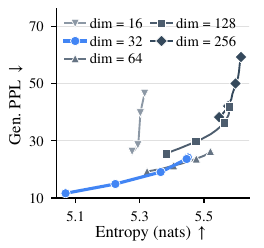}
        \par\small\textbf{(d)}
    \end{minipage}\hfill
    \begin{minipage}[t]{0.32\textwidth}
        \centering
        \includegraphics[width=\linewidth]{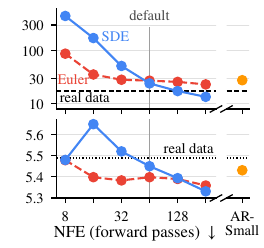}
        \par\small\textbf{(e)}
    \end{minipage}\hfill
    \begin{minipage}[t]{0.32\textwidth}
        \centering
        \includegraphics[width=\linewidth]{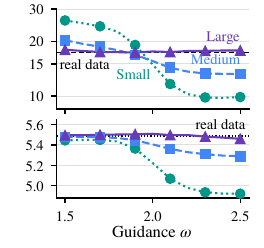}
        \par\small\textbf{(f)}
    \end{minipage}\\[4pt]
    \includegraphics{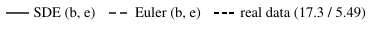}
    \caption{\textbf{Ablations of \dflow.} Results use \dflow-Small at 64 NFE unless otherwise indicated. \textbf{(a)} Singular logits under momentum guidance $\alpha$. \textbf{(b)} Embedding matrix. \textbf{(c)} CE and MSE training objectives over guidance scale $\omega$. \textbf{(d)} Flow dimension $d$. \textbf{(e)} NFE at $\omega=1.7$. \textbf{(f)} Guidance response for Small, Medium, and Large models. Panels (a)--(d) show Gen.\ PPL versus entropy; (e)--(f) show Gen.\ PPL above and entropy below. Black reference lines in (e)--(f) denote real text. At entropy $5.45$, CE reaches Gen.\ PPL $24.1$ versus $25.6$ for MSE. Larger models lose less entropy as guidance increases. Each point uses 1,024 sequences unless noted.}
    \label{fig:ablations}
\end{figure}

\paragraph{Singular Logits.}
Singular logits improve the quality--diversity trade-off (Figure~\ref{fig:ablations}(a)). Without guidance, the model with singular logits reaches Gen.\ PPL $57.2$ at entropy $5.48$, versus $64.2$ at $5.48$ for the model without singular logits, and the curve with singular logits lies below the curve without them across their shared entropy range. The gap widens under strong guidance: from $\alpha=0.5$ to $0.7$, Gen.\ PPL falls from $29.3$ to $24.0$ with singular logits but only from $37.2$ to $36.3$ without them.

\paragraph{Embedding Matrix.}
Hellinger--SVD outperforms random Gaussian and one-hot embeddings, all with singular logits (Figure~\ref{fig:ablations}(b)). With SDE sampling at $\alpha=0.2$, Hellinger--SVD reaches Gen.\ PPL $42.1$ at entropy $5.37$, compared with $83.6$ at $5.34$ for a 32-dimensional random Gaussian embedding. With Euler sampling, Hellinger--SVD attains lower Gen.\ PPL and higher entropy than the one-hot embedding at every guidance strength, even though the one-hot model has about 215M parameters versus 138M for the Hellinger--SVD model.

\paragraph{Training Objective.}
Training with CE yields a better quality--diversity trade-off than training with MSE (Figure~\ref{fig:ablations}(c)). The MSE variant regresses the velocity $\vv_\theta=\tfrac{1}{1-t}(\widehat{X}_1-X_t)$ onto $X_1-X_0$, as in ELF~\citep{elf2026}, and applies the distilled-guidance shift to this velocity target rather than to the logits. The two models differ only in the training loss. At matched entropy, CE attains lower Gen.\ PPL: $24.1$ at entropy $5.45$ ($\omega=1.7$) versus $25.6$ at $5.45$ for MSE ($\omega=2.8$).

\paragraph{Flow Embedding Dimension.}
The default 32-dimensional flow embedding space gives the best quality--diversity trade-off (Figure~\ref{fig:ablations}(d)). We vary $d\in\{16,32,64,128,256\}$, retaining the leading $d$ Hellinger--SVD directions (Section~\ref{sec:embedding_training}) and normalizing the unigram-weighted mean squared norm to $d$; only the input projection depends on $d$. Larger dimensions shift the curve toward higher entropy and need much stronger guidance to lower Gen.\ PPL. Even at $\omega=2.8$ and $2.6$, $d=128$ and $d=256$ remain at Gen.\ PPL $27.3$ and $42.1$ (entropies $5.42$ and $5.58$). $d=64$ needs $\omega=2.4$ to reach $23.7$ at $5.48$, whereas $d=32$ reaches $23.6$ at $5.45$ with $\omega=1.8$. Reducing the dimension to $d=16$ also hurts: its curve lies above that of $d=32$. Table~\ref{tab:flow_dim_sweep} in Appendix~\ref{app:ablation_details} lists the plotted values.

\experimentfinding{Singular logits, Hellinger--SVD embeddings, and CE training each lower Gen.\ PPL at matched entropy, and a 32-dimensional flow space gives the best quality--diversity trade-off.}

\paragraph{Sampling Budget.}
Figure~\ref{fig:ablations}(e) varies only the number of steps within a fixed grid family (Appendix~\ref{app:b_sampling_schedule}) at the default guidance scale $\omega=1.7$. Gen.\ PPL drops sharply as NFE grows to the default 64 and decreases further at 256 NFE, with lower entropy.

\paragraph{Guidance and Model Size.}
Figure~\ref{fig:ablations}(f) compares the guidance response across model sizes. Larger models lose less entropy as guidance increases; \dflow-Large maintains nearly constant Gen.\ PPL and entropy across the sweep. Figure~\ref{fig:pareto_by_size} in Appendix~\ref{app:additional_results} compares these responses with the AR and ELF quality--diversity trade-offs.

\section{Related Work}\label{sec:related_works}

\paragraph{Discrete-State Models.}
Discrete-state models generate sequences by updating categorical variables, often at multiple positions per step.
Following masked-token pretraining~\citep{devlin2019bert} and iterative refinement~\citep{ghazvininejad2019mask}, D3PM formalized structured categorical corruption~\citep{austin2021structured}, while continuous-time Markov formulations established corresponding denoising dynamics~\citep{campbell2022continuous}.
Subsequent work develops discrete score estimation, simplified masked-denoising objectives, and time-independent parameterizations~\citep{lou2023discrete,sahoo2024simple,DBLP:conf/nips/ShiHWDT24,ou2024your}, alongside more general corruption kernels and probability paths~\citep{gat2024discrete,DBLP:conf/iclr/ShaulGH0SHKLC25,von2025generalized}.
The LLaDA family~\citep{nie2025large,zhu2025llada,DBLP:journals/corr/abs-2512-15745} advances masked diffusion through scaling and preference optimization, alongside work on language and code generation~\citep{DBLP:journals/corr/abs-2508-15487,DBLP:journals/corr/abs-2506-20639,song2025seed} and uniform-state diffusion~\citep{DBLP:conf/icml/SahooDGWCK25,DBLP:journals/corr/abs-2602-15014,DBLP:journals/corr/abs-2606-19005}.
Sampling research develops blockwise decoding~\citep{DBLP:conf/iclr/ArriolaGCYQHSK25}, remasking~\citep{DBLP:conf/nips/WangSSK25}, and adaptive token ordering and parallel updates~\citep{kim2025train,DBLP:journals/corr/abs-2505-22618,DBLP:conf/nips/Ben-HamuGSNK25}.
Reinforcement learning supports reasoning~\citep{zhao2025d1}, while masked generation extends to images and multimodal modeling~\citep{chang2022maskgit,DBLP:conf/nips/YangTLZSTW25}.
These models retain categorical states even when time is continuous.

\paragraph{Continuous-State Models.}
Continuous-state models generate categorical sequences through real-valued representations, with early examples including Argmax Flows~\citep{hoogeboom2021argmax}.
Diffusion-LM~\citep{li2022diffusion} used token-embedding diffusion for controllable text generation, while Analog Bits~\citep{chen2022analog} introduced self-conditioning for diffusion over continuous relaxations of binary encodings.
Subsequent embedding-space models incorporate self-conditioning~\citep{DBLP:journals/corr/abs-2211-04236}, continuous-time formulations~\citep{DBLP:journals/corr/abs-2211-15089}, conditional generation and pretraining~\citep{DBLP:conf/iclr/GongLF0K23,DBLP:conf/icml/LinGSWFLDC23,ye2023dinoiser}, and likelihood-based training~\citep{gulrajani2023likelihood}.
Alternative representations include vocabulary-sized logit/simplex states~\citep{han2022ssd,DBLP:conf/eacl/MahabadiITHBPC24,DBLP:conf/acl/TaeIKC25} and statistical-manifold constructions~\citep{DBLP:conf/nips/JoH25,DBLP:conf/nips/ChengL0L24,DBLP:journals/corr/abs-2504-10283}.
Latent diffusion denoises encoded text representations with a separate decoder~\citep{lovelace2023latent,DBLP:conf/nips/0002G0ZSJ23,DBLP:conf/aaai/ShabalinMCLKBMM25,DBLP:conf/nips/MeshchaninovCSA25,DBLP:journals/corr/abs-2603-02547}, with recent work jointly learning the latent space and generative dynamics~\citep{DBLP:journals/corr/abs-2605-07933,DBLP:journals/corr/abs-2605-06548}.
Rectified flow has also been applied to conditional text generation~\citep{DBLP:conf/eacl/HuWAMFOS24}, while continuous flow-map language models pursue few-step decoding through distillation~\citep{lee2026flow,DBLP:journals/corr/abs-2602-12233,DBLP:journals/corr/abs-2604-09784}.
LangFlow~\citep{chen2026langflow} uses a Bregman formulation to train embedding-space flows with cross-entropy, whereas ELF~\citep{elf2026} performs rectified flow in pretrained contextual embeddings and applies token-level supervision only at final decoding.
Within this family, \dflow combines analytic singular logits for endpoint commitment with a compact, context-informed flow embedding space under a single cross-entropy objective, and reaches few-step generation through regularized reflow rather than distillation.

\section{Conclusion}
\label{sec:conclusion}

We introduced \dflow, a continuous rectified-flow language model trained with a cross-entropy loss. Its singular logits capture the divergent part of the token posteriors analytically, so trajectories converge to token embeddings, while the bounded learned residual decides which tokens are generated. Fixed Hellinger--SVD embeddings place tokens with similar contexts nearby, self-conditioning guidance in logit space is distilled through the same CE objective, and regularized reflow straightens the flow for few-step generation. On OpenWebText at 64 NFE, \dflow outperforms ELF and autoregressive baselines of the same size, and \dflow-Large matches real text in Gen.\ PPL at similar entropy. After two reflow rounds, \dflow keeps its advantage over same-size autoregressive models at 16 NFE while generating 15--19$\times$ faster at batch size 1. These results show that continuous flows can match autoregressive quality while offering fast, parallel generation.

\subsubsection*{Acknowledgements}

This work was supported in part by the Institute for Foundations of Machine Learning (IFML). The authors acknowledge the Texas Advanced Computing Center (TACC) at The University of Texas at Austin for providing computational resources that have contributed to the research results reported within this paper.



\bibliography{references}

@inproceedings{lebret2014hellinger,
  title={Word Embeddings through {Hellinger PCA}},
  author={Lebret, R{\'e}mi and Collobert, Ronan},
  booktitle={Proceedings of the 14th Conference of the European Chapter of the Association for Computational Linguistics},
  pages={482--490},
  year={2014},
  publisher={Association for Computational Linguistics},
  url={https://aclanthology.org/E14-1051/}
}

@inproceedings{liu2022flow,
  title        = {Flow Straight and Fast: Learning to Generate and Transfer Data with Rectified Flow},
  author       = {Liu, Xingchao and Gong, Chengyue and Liu, Qiang},
  booktitle    = {The Eleventh International Conference on Learning Representations, {ICLR}
                  2023},
  publisher    = {OpenReview.net},
  year         = {2023},
  url          = {https://openreview.net/forum?id=XVjTT1nw5z}
}

@inproceedings{wortsman2022robust,
  title        = {Robust Fine-Tuning of Zero-Shot Models},
  author       = {Wortsman, Mitchell and Ilharco, Gabriel and Kim, Jong Wook and Li, Mike and Kornblith, Simon and Roelofs, Rebecca and Gontijo-Lopes, Raphael and Hajishirzi, Hannaneh and Farhadi, Ali and Namkoong, Hongseok and Schmidt, Ludwig},
  booktitle    = {{IEEE/CVF} Conference on Computer Vision and Pattern Recognition, {CVPR}
                  2022},
  pages        = {7959--7971},
  publisher    = {{IEEE}},
  year         = {2022},
  url          = {https://openaccess.thecvf.com/content/CVPR2022/html/Wortsman_Robust_Fine-Tuning_of_Zero-Shot_Models_CVPR_2022_paper.html}
}

@inproceedings{lipman2022flow,
  title        = {Flow Matching for Generative Modeling},
  author       = {Lipman, Yaron and Chen, Ricky T. Q. and Ben-Hamu, Heli and Nickel, Maximilian and Le, Matt},
  booktitle    = {The Eleventh International Conference on Learning Representations, {ICLR}
                  2023},
  publisher    = {OpenReview.net},
  year         = {2023},
  url          = {https://openreview.net/forum?id=PqvMRDCJT9t}
}

@inproceedings{devlin2019bert,
  author       = {Jacob Devlin and
                  Ming{-}Wei Chang and
                  Kenton Lee and
                  Kristina Toutanova},
  title        = {{BERT}: Pre-training of Deep Bidirectional Transformers for Language
                  Understanding},
  booktitle    = {Proceedings of the 2019 Conference of the North American Chapter of the
                  Association for Computational Linguistics: Human Language
                  Technologies, {NAACL-HLT} 2019, Volume 1 (Long and Short
                  Papers)},
  pages        = {4171--4186},
  publisher    = {Association for Computational Linguistics},
  year         = {2019},
  url          = {https://aclanthology.org/N19-1423/},
  bibsource    = {dblp computer science bibliography, https://dblp.org}
}

@inproceedings{ghazvininejad2019mask,
  author       = {Marjan Ghazvininejad and
                  Omer Levy and
                  Yinhan Liu and
                  Luke Zettlemoyer},
  title        = {{Mask-Predict}: Parallel Decoding of Conditional Masked Language Models},
  booktitle    = {Proceedings of the 2019 Conference on Empirical Methods in Natural Language
                  Processing and the 9th International Joint Conference on
                  Natural Language Processing, {EMNLP-IJCNLP} 2019},
  pages        = {6112--6121},
  publisher    = {Association for Computational Linguistics},
  year         = {2019},
  url          = {https://aclanthology.org/D19-1633/},
  bibsource    = {dblp computer science bibliography, https://dblp.org}
}

@inproceedings{hoogeboom2021argmax,
  author       = {Emiel Hoogeboom and
                  Didrik Nielsen and
                  Priyank Jaini and
                  Patrick Forr{\'{e}} and
                  Max Welling},
  title        = {Argmax Flows and Multinomial Diffusion: Learning Categorical Distributions},
  booktitle    = {Advances in Neural Information Processing Systems 34: Annual Conference on
                  Neural Information Processing Systems 2021, NeurIPS 2021},
  pages        = {12454--12465},
  year         = {2021},
  url          = {https://proceedings.neurips.cc/paper/2021/hash/67d96d458abdef21792e6d8e590244e7-Abstract.html},
  bibsource    = {dblp computer science bibliography, https://dblp.org}
}

@inproceedings{austin2021structured,
  author       = {Jacob Austin and
                  Daniel D. Johnson and
                  Jonathan Ho and
                  Daniel Tarlow and
                  Rianne van den Berg},
  title        = {Structured Denoising Diffusion Models in Discrete State-Spaces},
  booktitle    = {Advances in Neural Information Processing Systems 34: Annual Conference on
                  Neural Information Processing Systems 2021, NeurIPS 2021},
  pages        = {17981--17993},
  year         = {2021},
  url          = {https://proceedings.neurips.cc/paper/2021/hash/958c530554f78bcd8e97125b70e6973d-Abstract.html},
  bibsource    = {dblp computer science bibliography, https://dblp.org}
}

@inproceedings{campbell2022continuous,
  author       = {Andrew Campbell and
                  Joe Benton and
                  Valentin {De Bortoli} and
                  Thomas Rainforth and
                  George Deligiannidis and
                  Arnaud Doucet},
  title        = {A Continuous Time Framework for Discrete Denoising Models},
  booktitle    = {Advances in Neural Information Processing Systems 35: Annual Conference on
                  Neural Information Processing Systems 2022, NeurIPS 2022},
  year         = {2022},
  url          = {http://papers.nips.cc/paper\_files/paper/2022/hash/b5b528767aa35f5b1a60fe0aaeca0563-Abstract-Conference.html},
  bibsource    = {dblp computer science bibliography, https://dblp.org}
}

@inproceedings{li2022diffusion,
  author       = {Xiang Lisa Li and
                  John Thickstun and
                  Ishaan Gulrajani and
                  Percy Liang and
                  Tatsunori B. Hashimoto},
  title        = {{Diffusion-LM} Improves Controllable Text Generation},
  booktitle    = {Advances in Neural Information Processing Systems 35: Annual Conference on
                  Neural Information Processing Systems 2022, NeurIPS 2022},
  year         = {2022},
  url          = {http://papers.nips.cc/paper\_files/paper/2022/hash/1be5bc25d50895ee656b8c2d9eb89d6a-Abstract-Conference.html},
  bibsource    = {dblp computer science bibliography, https://dblp.org}
}

@inproceedings{gulrajani2023likelihood,
  author       = {Ishaan Gulrajani and
                  Tatsunori B. Hashimoto},
  title        = {Likelihood-Based Diffusion Language Models},
  booktitle    = {Advances in Neural Information Processing Systems 36: Annual Conference on
                  Neural Information Processing Systems 2023, NeurIPS 2023},
  year         = {2023},
  url          = {http://papers.nips.cc/paper\_files/paper/2023/hash/35b5c175e139bff5f22a5361270fce87-Abstract-Conference.html},
  bibsource    = {dblp computer science bibliography, https://dblp.org}
}

@inproceedings{han2022ssd,
  author       = {Xiaochuang Han and
                  Sachin Kumar and
                  Yulia Tsvetkov},
  title        = {{SSD-LM}: Semi-autoregressive Simplex-based Diffusion Language Model
                  for Text Generation and Modular Control},
  booktitle    = {Proceedings of the 61st Annual Meeting of the Association for Computational
                  Linguistics (Volume 1: Long Papers), {ACL} 2023},
  pages        = {11575--11596},
  publisher    = {Association for Computational Linguistics},
  year         = {2023},
  url          = {https://aclanthology.org/2023.acl-long.647/},
  bibsource    = {dblp computer science bibliography, https://dblp.org}
}

@inproceedings{sahoo2024simple,
  author       = {Subham Sekhar Sahoo and
                  Marianne Arriola and
                  Yair Schiff and
                  Aaron Gokaslan and
                  Edgar Marroquin and
                  Justin T. Chiu and
                  Alexander Rush and
                  Volodymyr Kuleshov},
  title        = {Simple and Effective Masked Diffusion Language Models},
  booktitle    = {Advances in Neural Information Processing Systems 37: Annual Conference on
                  Neural Information Processing Systems 2024, NeurIPS 2024},
  year         = {2024},
  url          = {http://papers.nips.cc/paper\_files/paper/2024/hash/eb0b13cc515724ab8015bc978fdde0ad-Abstract-Conference.html},
  bibsource    = {dblp computer science bibliography, https://dblp.org}
}

@inproceedings{nie2025large,
  author       = {Shen Nie and
                  Fengqi Zhu and
                  Zebin You and
                  Xiaolu Zhang and
                  Jingyang Ou and
                  Jun Hu and
                  Jun Zhou and
                  Yankai Lin and
                  Ji{-}Rong Wen and
                  Chongxuan Li},
  title        = {Large Language Diffusion Models},
  booktitle    = {Advances in Neural Information Processing Systems 38: Annual Conference on
                  Neural Information Processing Systems 2025, NeurIPS 2025},
  year         = {2025},
  url          = {http://papers.nips.cc/paper\_files/paper/2025/hash/48b383b24230e0e6e649d9c98dae4d8c-Abstract-Conference.html},
  bibsource    = {dblp computer science bibliography, https://dblp.org}
}

@inproceedings{gat2024discrete,
  author       = {Itai Gat and
                  Tal Remez and
                  Neta Shaul and
                  Felix Kreuk and
                  Ricky T. Q. Chen and
                  Gabriel Synnaeve and
                  Yossi Adi and
                  Yaron Lipman},
  title        = {Discrete Flow Matching},
  booktitle    = {Advances in Neural Information Processing Systems 37: Annual Conference on
                  Neural Information Processing Systems 2024, NeurIPS 2024},
  year         = {2024},
  url          = {http://papers.nips.cc/paper\_files/paper/2024/hash/f0d629a734b56a642701bba7bc8bb3ed-Abstract-Conference.html},
  bibsource    = {dblp computer science bibliography, https://dblp.org}
}

@inproceedings{ou2024your,
  author       = {Jingyang Ou and
                  Shen Nie and
                  Kaiwen Xue and
                  Fengqi Zhu and
                  Jiacheng Sun and
                  Zhenguo Li and
                  Chongxuan Li},
  title        = {Your Absorbing Discrete Diffusion Secretly Models the Conditional
                  Distributions of Clean Data},
  booktitle    = {The Thirteenth International Conference on Learning Representations, {ICLR}
                  2025},
  publisher    = {OpenReview.net},
  year         = {2025},
  url          = {https://openreview.net/forum?id=sMyXP8Tanm},
  bibsource    = {dblp computer science bibliography, https://dblp.org}
}

@inproceedings{chang2022maskgit,
  author       = {Huiwen Chang and
                  Han Zhang and
                  Lu Jiang and
                  Ce Liu and
                  William T. Freeman},
  title        = {{MaskGIT}: Masked Generative Image Transformer},
  booktitle    = {{IEEE/CVF} Conference on Computer Vision and Pattern Recognition, {CVPR}
                  2022},
  pages        = {11305--11315},
  publisher    = {{IEEE}},
  year         = {2022},
  url          = {https://openaccess.thecvf.com/content/CVPR2022/html/Chang_MaskGIT_Masked_Generative_Image_Transformer_CVPR_2022_paper.html},
  bibsource    = {dblp computer science bibliography, https://dblp.org}
}

@inproceedings{von2025generalized,
  author       = {Dimitri von R{\"{u}}tte and
                  Janis Fluri and
                  Yuhui Ding and
                  Antonio Orvieto and
                  Bernhard Sch{\"{o}}lkopf and
                  Thomas Hofmann},
  title        = {Generalized Interpolating Discrete Diffusion},
  booktitle    = {Forty-second International Conference on Machine Learning, {ICML} 2025},
  series       = {Proceedings of Machine Learning Research},
  volume       = {267},
  pages        = {61810--61843},
  publisher    = {{PMLR}},
  year         = {2025},
  url          = {https://proceedings.mlr.press/v267/von-rutte25a.html},
  bibsource    = {dblp computer science bibliography, https://dblp.org}
}

@inproceedings{kim2025train,
  author       = {Jaeyeon Kim and
                  Kulin Shah and
                  Vasilis Kontonis and
                  Sham M. Kakade and
                  Sitan Chen},
  title        = {Train for the Worst, Plan for the Best: Understanding Token Ordering
                  in Masked Diffusions},
  booktitle    = {Forty-second International Conference on Machine Learning, {ICML} 2025},
  series       = {Proceedings of Machine Learning Research},
  volume       = {267},
  pages        = {30749--30768},
  publisher    = {{PMLR}},
  year         = {2025},
  url          = {https://proceedings.mlr.press/v267/kim25ah.html},
  bibsource    = {dblp computer science bibliography, https://dblp.org}
}

@article{song2025seed,
  author       = {Yuxuan Song and
                  Zheng Zhang and
                  Cheng Luo and
                  Pengyang Gao and
                  Fan Xia and
                  Hao Luo and
                  Zheng Li and
                  Yuehang Yang and
                  Hongli Yu and
                  Xingwei Qu and
                  Yuwei Fu and
                  Jing Su and
                  Ge Zhang and
                  Wenhao Huang and
                  Mingxuan Wang and
                  Lin Yan and
                  Xiaoying Jia and
                  Jingjing Liu and
                  Wei{-}Ying Ma and
                  Ya{-}Qin Zhang and
                  Yonghui Wu and
                  Hao Zhou},
  title        = {{Seed Diffusion}: {A} Large-Scale Diffusion Language Model with High-Speed
                  Inference},
  journal      = {CoRR},
  volume       = {abs/2508.02193},
  year         = {2025},
  url          = {https://arxiv.org/abs/2508.02193},
  eprinttype   = {arXiv},
  eprint       = {2508.02193},
  bibsource    = {dblp computer science bibliography, https://dblp.org}
}

@inproceedings{zhu2025llada,
  author       = {Fengqi Zhu and
                  Rongzhen Wang and
                  Shen Nie and
                  Xiaolu Zhang and
                  Chunwei Wu and
                  Jun Zhou and
                  Yankai Lin and
                  Ji{-}Rong Wen and
                  Chongxuan Li},
  title        = {{LLaDA} 1.5: Variance-Reduced Preference Optimization for Large Language
                  Diffusion Models},
  booktitle    = {Proceedings of the 64th Annual Meeting of the Association for Computational
                  Linguistics (Volume 1: Long Papers), {ACL} 2026},
  pages        = {11425--11460},
  publisher    = {Association for Computational Linguistics},
  year         = {2026},
  doi          = {10.18653/v1/2026.acl-long.524},
  url          = {https://aclanthology.org/2026.acl-long.524/}
}

@inproceedings{zhao2025d1,
  author       = {Siyan Zhao and
                  Devaansh Gupta and
                  Qinqing Zheng and
                  Aditya Grover},
  title        = {d1: Scaling Reasoning in Diffusion Large Language Models via Reinforcement
                  Learning},
  booktitle    = {Advances in Neural Information Processing Systems 38: Annual Conference on
                  Neural Information Processing Systems 2025, NeurIPS 2025},
  year         = {2025},
  url          = {http://papers.nips.cc/paper\_files/paper/2025/hash/52190a0362148d179f1cbd9080956872-Abstract-Conference.html},
  bibsource    = {dblp computer science bibliography, https://dblp.org}
}

@inproceedings{lou2023discrete,
  author       = {Aaron Lou and
                  Chenlin Meng and
                  Stefano Ermon},
  title        = {Discrete Diffusion Modeling by Estimating the Ratios of the Data Distribution},
  booktitle    = {Forty-first International Conference on Machine Learning, {ICML} 2024},
  series       = {Proceedings of Machine Learning Research},
  volume       = {235},
  pages        = {32819--32848},
  publisher    = {{PMLR} / OpenReview.net},
  year         = {2024},
  url          = {https://proceedings.mlr.press/v235/lou24a.html},
  bibsource    = {dblp computer science bibliography, https://dblp.org}
}

@inproceedings{lovelace2023latent,
  author       = {Justin Lovelace and
                  Varsha Kishore and
                  Chao Wan and
                  Eliot Shekhtman and
                  Kilian Q. Weinberger},
  title        = {Latent Diffusion for Language Generation},
  booktitle    = {Advances in Neural Information Processing Systems 36: Annual Conference on
                  Neural Information Processing Systems 2023, NeurIPS 2023},
  year         = {2023},
  url          = {http://papers.nips.cc/paper\_files/paper/2023/hash/b2a2bd5d5051ff6af52e1ef60aefd255-Abstract-Conference.html},
  bibsource    = {dblp computer science bibliography, https://dblp.org}
}

@article{su2024roformer,
  author       = {Jianlin Su and
                  Murtadha H. M. Ahmed and
                  Yu Lu and
                  Shengfeng Pan and
                  Wen Bo and
                  Yunfeng Liu},
  title        = {{RoFormer}: Enhanced Transformer with Rotary Position Embedding},
  journal      = {Neurocomputing},
  volume       = {568},
  pages        = {127063},
  year         = {2024},
  url          = {https://linkinghub.elsevier.com/retrieve/pii/S0925231223011864},
  bibsource    = {dblp computer science bibliography, https://dblp.org}
}

@inproceedings{qiu2025gated,
  author       = {Zihan Qiu and
                  Zekun Wang and
                  Bo Zheng and
                  Zeyu Huang and
                  Kaiyue Wen and
                  Songlin Yang and
                  Rui Men and
                  Le Yu and
                  Fei Huang and
                  Suozhi Huang and
                  Dayiheng Liu and
                  Jingren Zhou and
                  Junyang Lin},
  title        = {Gated Attention for Large Language Models: Non-linearity, Sparsity,
                  and Attention-Sink-Free},
  booktitle    = {Advances in Neural Information Processing Systems 38: Annual Conference on
                  Neural Information Processing Systems 2025, NeurIPS 2025},
  year         = {2025},
  url          = {http://papers.nips.cc/paper\_files/paper/2025/hash/904e89bb4e632e75fb47f093b620b257-Abstract-Conference.html},
  bibsource    = {dblp computer science bibliography, https://dblp.org}
}

@article{elf2026,
  author       = {Keya Hu and
                  Linlu Qiu and
                  Yiyang Lu and
                  Hanhong Zhao and
                  Tianhong Li and
                  Yoon Kim and
                  Jacob Andreas and
                  Kaiming He},
  title        = {{ELF}: Embedded Language Flows},
  journal      = {CoRR},
  volume       = {abs/2605.10938},
  year         = {2026},
  url          = {https://arxiv.org/abs/2605.10938},
  eprinttype   = {arXiv},
  eprint       = {2605.10938},
  bibsource    = {dblp computer science bibliography, https://dblp.org}
}

@inproceedings{dao2022flashattention,
  title        = {Flash{A}ttention: Fast and Memory-Efficient Exact Attention with {IO}-Awareness},
  author       = {Dao, Tri and Fu, Daniel Y. and Ermon, Stefano and Rudra, Atri and R{\'e}, Christopher},
  booktitle    = {Advances in Neural Information Processing Systems 35: Annual Conference on
                  Neural Information Processing Systems 2022, NeurIPS 2022},
  year         = {2022},
  url          = {https://proceedings.neurips.cc/paper_files/paper/2022/hash/67d57c32e20fd0a7a302cb81d36e40d5-Abstract-Conference.html}
}

@inproceedings{dao2023flashattention2,
  title        = {Flash{A}ttention-2: Faster Attention with Better Parallelism and Work Partitioning},
  author       = {Dao, Tri},
  booktitle    = {The Twelfth International Conference on Learning Representations, {ICLR}
                  2024},
  publisher    = {OpenReview.net},
  year         = {2024},
  url          = {https://openreview.net/forum?id=mZn2Xyh9Ec}
}

@misc{gokaslan2019OpenWeb,
  title        = {{OpenWebText} Corpus},
  author       = {Gokaslan, Aaron and Cohen, Vanya},
  howpublished = {\url{http://Skylion007.github.io/OpenWebTextCorpus}},
  year         = {2019}
}

@techreport{radford2019language,
  title        = {Language Models are Unsupervised Multitask Learners},
  author       = {Radford, Alec and Wu, Jeffrey and Child, Rewon and Luan, David and Amodei, Dario and Sutskever, Ilya},
  institution  = {OpenAI},
  year         = {2019},
  url          = {https://cdn.openai.com/better-language-models/language_models_are_unsupervised_multitask_learners.pdf}
}

@misc{jordan2024muon,
  author       = {Keller Jordan and Yuchen Jin and Vlado Boza and Jiacheng You and
                  Franz Cesista and Laker Newhouse and Jeremy Bernstein},
  title        = {Muon: An optimizer for hidden layers in neural networks},
  year         = {2024},
  url          = {https://kellerjordan.github.io/posts/muon/}
}

@article{Chen2025MuonOU,
  title={Muon Optimizes Under Spectral Norm Constraints},
  author={Li-Zhang Chen and Jonathan Li and Qiang Liu},
  journal={Trans. Mach. Learn. Res.},
  year={2025},
  volume={2026},
  url={https://api.semanticscholar.org/CorpusID:279447438}
}

@inproceedings{Chen2024LionSS,
  title={Lion Secretly Solves a Constrained Optimization: As Lyapunov Predicts},
  author={Li-Zhang Chen and Bo Liu and Kai-Zhao Liang and Qiang Liu},
  booktitle={International Conference on Learning Representations},
  year={2024},
  url={https://api.semanticscholar.org/CorpusID:271532744}
}

@article{liao2026momentum,
  title        = {Momentum Guidance: Plug-and-Play Guidance for Flow Models},
  author       = {Liao, Runlong and Yu, Jian and Su, Baiyu and Zhang, Chi and Chen, Lizhang and Liu, Qiang},
  journal      = {CoRR},
  volume       = {abs/2602.20360},
  year         = {2026},
  url          = {https://arxiv.org/abs/2602.20360}
}

@inproceedings{chen2022analog,
  title        = {Analog Bits: Generating Discrete Data using Diffusion Models with Self-Conditioning},
  author       = {Chen, Ting and Zhang, Ruixiang and Hinton, Geoffrey},
  booktitle    = {The Eleventh International Conference on Learning Representations, {ICLR}
                  2023},
  publisher    = {OpenReview.net},
  year         = {2023},
  url          = {https://openreview.net/forum?id=3itjR9QxFw}
}

@article{chen2026langflow,
  author       = {Yuxin Chen and
                  Chumeng Liang and
                  Hangke Sui and
                  Ruihan Guo and
                  Chaoran Cheng and
                  Jiaxuan You and
                  Ge Liu},
  title        = {{LangFlow}: Continuous Diffusion Rivals Discrete in Language Modeling},
  journal      = {CoRR},
  volume       = {abs/2604.11748},
  year         = {2026},
  url          = {https://arxiv.org/abs/2604.11748},
  eprinttype   = {arXiv},
  eprint       = {2604.11748},
  bibsource    = {dblp computer science bibliography, https://dblp.org}
}

@inproceedings{ma2024sit,
  title={{SiT}: Exploring flow and diffusion-based generative models with scalable interpolant transformers},
  author={Ma, Nanye and Goldstein, Mark and Albergo, Michael S. and Boffi, Nicholas M. and Vanden-Eijnden, Eric and Xie, Saining},
  booktitle={Computer Vision -- {ECCV} 2024},
  series={Lecture Notes in Computer Science},
  volume={15135},
  pages={23--40},
  publisher={Springer},
  year={2024},
  doi={10.1007/978-3-031-72980-5_2}
}

@article{DBLP:journals/corr/abs-2211-15089,
  author       = {Sander Dieleman and
                  Laurent Sartran and
                  Arman Roshannai and
                  Nikolay Savinov and
                  Yaroslav Ganin and
                  Pierre H. Richemond and
                  Arnaud Doucet and
                  Robin Strudel and
                  Chris Dyer and
                  Conor Durkan and
                  Curtis Hawthorne and
                  R{\'{e}}mi Leblond and
                  Will Grathwohl and
                  Jonas Adler},
  title        = {Continuous diffusion for categorical data},
  journal      = {CoRR},
  volume       = {abs/2211.15089},
  year         = {2022},
  url          = {https://arxiv.org/abs/2211.15089},
  eprinttype   = {arXiv},
  eprint       = {2211.15089},
  bibsource    = {dblp computer science bibliography, https://dblp.org}
}

@article{DBLP:journals/corr/abs-2211-04236,
  author       = {Robin Strudel and
                  Corentin Tallec and
                  Florent Altch{\'{e}} and
                  Yilun Du and
                  Yaroslav Ganin and
                  Arthur Mensch and
                  Will Grathwohl and
                  Nikolay Savinov and
                  Sander Dieleman and
                  Laurent Sifre and
                  R{\'{e}}mi Leblond},
  title        = {Self-conditioned Embedding Diffusion for Text Generation},
  journal      = {CoRR},
  volume       = {abs/2211.04236},
  year         = {2022},
  url          = {https://arxiv.org/abs/2211.04236},
  eprinttype   = {arXiv},
  eprint       = {2211.04236},
  bibsource    = {dblp computer science bibliography, https://dblp.org}
}

@inproceedings{DBLP:conf/iclr/GongLF0K23,
  author       = {Shansan Gong and
                  Mukai Li and
                  Jiangtao Feng and
                  Zhiyong Wu and
                  Lingpeng Kong},
  title        = {{DiffuSeq}: Sequence to Sequence Text Generation with Diffusion Models},
  booktitle    = {The Eleventh International Conference on Learning Representations, {ICLR}
                  2023},
  publisher    = {OpenReview.net},
  year         = {2023},
  url          = {https://openreview.net/forum?id=jQj-\_rLVXsj},
  bibsource    = {dblp computer science bibliography, https://dblp.org}
}

@inproceedings{DBLP:conf/icml/LinGSWFLDC23,
  author       = {Zhenghao Lin and
                  Yeyun Gong and
                  Yelong Shen and
                  Tong Wu and
                  Zhihao Fan and
                  Chen Lin and
                  Nan Duan and
                  Weizhu Chen},
  title        = {Text Generation with Diffusion Language Models: {A} Pre-training Approach
                  with Continuous Paragraph Denoise},
  booktitle    = {Fortieth International Conference on Machine Learning, {ICML} 2023},
  series       = {Proceedings of Machine Learning Research},
  volume       = {202},
  pages        = {21051--21064},
  publisher    = {{PMLR}},
  year         = {2023},
  url          = {https://proceedings.mlr.press/v202/lin23d.html},
  bibsource    = {dblp computer science bibliography, https://dblp.org}
}

@article{ye2023dinoiser,
  title        = {{DINOISER}: Diffused Conditional Sequence Learning by Manipulating Noises},
  author       = {Ye, Jiasheng and Zheng, Zaixiang and Bao, Yu and Qian, Lihua and Wang, Mingxuan},
  journal      = {Transactions of the Association for Computational Linguistics},
  year         = {2024},
  url          = {https://arxiv.org/abs/2302.10025}
}

@inproceedings{lee2026flow,
  title        = {Flow Map Language Models: One-step Language Modeling via Continuous Denoising},
  author       = {Lee, Chanhyuk and Yoo, Jaehoon and Agarwal, Manan and Shah, Sheel and Huang, Jerry and Raghunathan, Aditi and Hong, Seunghoon and Boffi, Nicholas M. and Kim, Jinwoo},
  booktitle    = {Advances in Neural Information Processing Systems 39: Annual Conference on
                  Neural Information Processing Systems 2026, NeurIPS 2026},
  year         = {2026},
  url          = {https://arxiv.org/abs/2602.16813}
}

@inproceedings{DBLP:journals/corr/abs-2602-12233,
  author       = {Daan Roos and
                  Oscar Davis and
                  Floor Eijkelboom and
                  Michael M. Bronstein and
                  Max Welling and
                  {\.I}smail {\.I}lkan Ceylan and
                  Luca Ambrogioni and
                  Jan{-}Willem van de Meent},
  title        = {Categorical Flow Maps},
  booktitle    = {Forty-third International Conference on Machine Learning, {ICML} 2026},
  series       = {Proceedings of Machine Learning Research},
  volume       = {306},
  pages        = {105618--105642},
  publisher    = {{PMLR}},
  year         = {2026},
  url          = {https://proceedings.mlr.press/v306/roos26a.html}
}

@article{DBLP:journals/corr/abs-2604-09784,
  author       = {Peter Potaptchik and
                  Jason Yim and
                  Adhi Saravanan and
                  Peter Holderrieth and
                  Eric Vanden{-}Eijnden and
                  Michael S. Albergo},
  title        = {Discrete Flow Maps},
  journal      = {CoRR},
  volume       = {abs/2604.09784},
  year         = {2026},
  url          = {https://arxiv.org/abs/2604.09784},
  eprinttype   = {arXiv},
  eprint       = {2604.09784},
  bibsource    = {dblp computer science bibliography, https://dblp.org}
}

@inproceedings{DBLP:conf/eacl/HuWAMFOS24,
  author       = {Vincent Tao Hu and
                  Di Wu and
                  Yuki Markus Asano and
                  Pascal Mettes and
                  Basura Fernando and
                  Bj{\"{o}}rn Ommer and
                  Cees Snoek},
  title        = {Flow Matching for Conditional Text Generation in a Few Sampling Steps},
  booktitle    = {Proceedings of the 18th Conference of the European Chapter of the
                  Association for Computational Linguistics, {EACL} 2024 -
                  Volume 2: Short Papers},
  pages        = {380--392},
  publisher    = {Association for Computational Linguistics},
  year         = {2024},
  url          = {https://aclanthology.org/2024.eacl-short.33/},
  bibsource    = {dblp computer science bibliography, https://dblp.org}
}

@inproceedings{DBLP:conf/eacl/MahabadiITHBPC24,
  author       = {Rabeeh {Karimi Mahabadi} and
                  Hamish Ivison and
                  Jaesung Tae and
                  James Henderson and
                  Iz Beltagy and
                  Matthew E. Peters and
                  Arman Cohan},
  title        = {{TESS}: Text-to-Text Self-Conditioned Simplex Diffusion},
  booktitle    = {Proceedings of the 18th Conference of the European Chapter of the
                  Association for Computational Linguistics, {EACL} 2024 -
                  Volume 1: Long Papers},
  pages        = {2347--2361},
  publisher    = {Association for Computational Linguistics},
  year         = {2024},
  url          = {https://aclanthology.org/2024.eacl-long.144/},
  bibsource    = {dblp computer science bibliography, https://dblp.org}
}

@inproceedings{DBLP:conf/acl/TaeIKC25,
  author       = {Jaesung Tae and
                  Hamish Ivison and
                  Sachin Kumar and
                  Arman Cohan},
  title        = {{TESS} 2: {A} Large-Scale Generalist Diffusion Language Model},
  booktitle    = {Proceedings of the 63rd Annual Meeting of the Association for Computational
                  Linguistics (Volume 1: Long Papers), {ACL} 2025},
  pages        = {21171--21188},
  publisher    = {Association for Computational Linguistics},
  year         = {2025},
  url          = {https://aclanthology.org/2025.acl-long.1029/},
  bibsource    = {dblp computer science bibliography, https://dblp.org}
}

@inproceedings{DBLP:conf/nips/JoH25,
  author       = {Jaehyeong Jo and
                  Sung Ju Hwang},
  title        = {Continuous Diffusion Model for Language Modeling},
  booktitle    = {Advances in Neural Information Processing Systems 38: Annual Conference on
                  Neural Information Processing Systems 2025, NeurIPS 2025},
  year         = {2025},
  url          = {http://papers.nips.cc/paper\_files/paper/2025/hash/8c4046a93293442e88f54e19a5849d68-Abstract-Conference.html},
  bibsource    = {dblp computer science bibliography, https://dblp.org}
}

@inproceedings{DBLP:conf/nips/ChengL0L24,
  author       = {Chaoran Cheng and
                  Jiahan Li and
                  Jian Peng and
                  Ge Liu},
  title        = {Categorical Flow Matching on Statistical Manifolds},
  booktitle    = {Advances in Neural Information Processing Systems 37: Annual Conference on
                  Neural Information Processing Systems 2024, NeurIPS 2024},
  year         = {2024},
  url          = {http://papers.nips.cc/paper\_files/paper/2024/hash/62a58f2130894e44e8a272c563a2c6f1-Abstract-Conference.html},
  bibsource    = {dblp computer science bibliography, https://dblp.org}
}

@article{DBLP:journals/corr/abs-2504-10283,
  author       = {Chaoran Cheng and
                  Jiahan Li and
                  Jiajun Fan and
                  Ge Liu},
  title        = {{\(\alpha\)-Flow}: {A} Unified Framework for Continuous-State Discrete
                  Flow Matching Models},
  journal      = {CoRR},
  volume       = {abs/2504.10283},
  year         = {2025},
  url          = {https://arxiv.org/abs/2504.10283},
  eprinttype   = {arXiv},
  eprint       = {2504.10283},
  bibsource    = {dblp computer science bibliography, https://dblp.org}
}

@inproceedings{DBLP:conf/nips/0002G0ZSJ23,
  author       = {Yizhe Zhang and
                  Jiatao Gu and
                  Zhuofeng Wu and
                  Shuangfei Zhai and
                  Joshua M. Susskind and
                  Navdeep Jaitly},
  title        = {{PLANNER}: Generating Diversified Paragraph via Latent Language Diffusion
                  Model},
  booktitle    = {Advances in Neural Information Processing Systems 36: Annual Conference on
                  Neural Information Processing Systems 2023, NeurIPS 2023},
  year         = {2023},
  url          = {http://papers.nips.cc/paper\_files/paper/2023/hash/fdba5e0a9b57fce03e89cc0cad0a24e9-Abstract-Conference.html},
  bibsource    = {dblp computer science bibliography, https://dblp.org}
}

@inproceedings{DBLP:conf/aaai/ShabalinMCLKBMM25,
  author       = {Alexander Shabalin and
                  Viacheslav Meshchaninov and
                  Egor Chimbulatov and
                  Vladislav Lapikov and
                  Roman Kim and
                  Grigory Bartosh and
                  Dmitry Molchanov and
                  Sergey Markov and
                  Dmitry P. Vetrov},
  title        = {{TEncDM}: Understanding the Properties of the Diffusion Model in the
                  Space of Language Model Encodings},
  booktitle    = {Thirty-Ninth {AAAI} Conference on Artificial Intelligence, Thirty-Seventh
                  Conference on Innovative Applications of Artificial
                  Intelligence, Fifteenth Symposium on Educational Advances
                  in Artificial Intelligence, {AAAI} 2025},
  pages        = {25110--25118},
  publisher    = {{AAAI} Press},
  year         = {2025},
  url          = {https://ojs.aaai.org/index.php/AAAI/article/view/34696},
  bibsource    = {dblp computer science bibliography, https://dblp.org}
}

@inproceedings{DBLP:conf/nips/MeshchaninovCSA25,
  author       = {Viacheslav Meshchaninov and
                  Egor Chimbulatov and
                  Alexander Shabalin and
                  Aleksandr Abramov and
                  Dmitry P. Vetrov},
  title        = {Cosmos: Compressed and Smooth Latent Space for Text Diffusion Modeling},
  booktitle    = {Advances in Neural Information Processing Systems 38: Annual Conference on
                  Neural Information Processing Systems 2025, NeurIPS 2025},
  year         = {2025},
  url          = {http://papers.nips.cc/paper\_files/paper/2025/hash/1506930fb75c82246e4d8648a66e4b27-Abstract-Conference.html},
  bibsource    = {dblp computer science bibliography, https://dblp.org}
}

@article{DBLP:journals/corr/abs-2603-02547,
  author       = {Junzhe Shen and
                  Jieru Zhao and
                  Ziwei He and
                  Zhouhan Lin},
  title        = {{CoDAR}: Continuous Diffusion Language Models are More Powerful Than
                  You Think},
  journal      = {CoRR},
  volume       = {abs/2603.02547},
  year         = {2026},
  url          = {https://arxiv.org/abs/2603.02547},
  eprinttype   = {arXiv},
  eprint       = {2603.02547},
  bibsource    = {dblp computer science bibliography, https://dblp.org}
}

@article{DBLP:journals/corr/abs-2605-07933,
  author       = {Viacheslav Meshchaninov and
                  Alexander Shabalin and
                  Egor Chimbulatov and
                  Nikita Gushchin and
                  Ilya Koziev and
                  Alexander Korotin and
                  Dmitry P. Vetrov},
  title        = {How to Train Your Latent Diffusion Language Model Jointly With the
                  Latent Space},
  journal      = {CoRR},
  volume       = {abs/2605.07933},
  year         = {2026},
  url          = {https://arxiv.org/abs/2605.07933},
  eprinttype   = {arXiv},
  eprint       = {2605.07933},
  bibsource    = {dblp computer science bibliography, https://dblp.org}
}

@article{DBLP:journals/corr/abs-2605-06548,
  author       = {Hongcan Guo and
                  Qinyu Zhao and
                  Yian Zhao and
                  Shen Nie and
                  Rui Zhu and
                  Qiushan Guo and
                  Feng Wang and
                  Tao Yang and
                  Hengshuang Zhao and
                  Guoqiang Wei and
                  Yan Zeng},
  title        = {Continuous Latent Diffusion Language Model},
  journal      = {CoRR},
  volume       = {abs/2605.06548},
  year         = {2026},
  url          = {https://arxiv.org/abs/2605.06548},
  eprinttype   = {arXiv},
  eprint       = {2605.06548},
  bibsource    = {dblp computer science bibliography, https://dblp.org}
}

@inproceedings{DBLP:conf/nips/ShiHWDT24,
  author       = {Jiaxin Shi and
                  Kehang Han and
                  Zhe Wang and
                  Arnaud Doucet and
                  Michalis K. Titsias},
  title        = {Simplified and Generalized Masked Diffusion for Discrete Data},
  booktitle    = {Advances in Neural Information Processing Systems 37: Annual Conference on
                  Neural Information Processing Systems 2024, NeurIPS 2024},
  year         = {2024},
  url          = {http://papers.nips.cc/paper\_files/paper/2024/hash/bad233b9849f019aead5e5cc60cef70f-Abstract-Conference.html},
  bibsource    = {dblp computer science bibliography, https://dblp.org}
}

@inproceedings{DBLP:conf/iclr/ShaulGH0SHKLC25,
  author       = {Neta Shaul and
                  Itai Gat and
                  Marton Havasi and
                  Daniel Severo and
                  Anuroop Sriram and
                  Peter Holderrieth and
                  Brian Karrer and
                  Yaron Lipman and
                  Ricky T. Q. Chen},
  title        = {Flow Matching with General Discrete Paths: {A} Kinetic-Optimal Perspective},
  booktitle    = {The Thirteenth International Conference on Learning Representations, {ICLR}
                  2025},
  publisher    = {OpenReview.net},
  year         = {2025},
  url          = {https://openreview.net/forum?id=tcvMzR2NrP},
  bibsource    = {dblp computer science bibliography, https://dblp.org}
}

@article{DBLP:journals/corr/abs-2508-15487,
  author       = {Jiacheng Ye and
                  Zhihui Xie and
                  Lin Zheng and
                  Jiahui Gao and
                  Zirui Wu and
                  Xin Jiang and
                  Zhenguo Li and
                  Lingpeng Kong},
  title        = {{Dream 7B}: Diffusion Large Language Models},
  journal      = {CoRR},
  volume       = {abs/2508.15487},
  year         = {2025},
  url          = {https://arxiv.org/abs/2508.15487},
  eprinttype   = {arXiv},
  eprint       = {2508.15487},
  bibsource    = {dblp computer science bibliography, https://dblp.org}
}

@article{DBLP:journals/corr/abs-2512-15745,
  author       = {Tiwei Bie and
                  Maosong Cao and
                  Kun Chen and
                  Lun Du and
                  Mingliang Gong and
                  Zhuochen Gong and
                  Yanmei Gu and
                  Jiaqi Hu and
                  Zenan Huang and
                  Zhenzhong Lan and
                  Chengxi Li and
                  Chongxuan Li and
                  Jianguo Li and
                  Zehuan Li and
                  Huabin Liu and
                  Lin Liu and
                  Guoshan Lu and
                  Xiaocheng Lu and
                  Yuxin Ma and
                  Jianfeng Tan and
                  Lanning Wei and
                  Ji{-}Rong Wen and
                  Yipeng Xing and
                  Xiaolu Zhang and
                  Junbo Zhao and
                  Da Zheng and
                  Jun Zhou and
                  Junlin Zhou and
                  Zhanchao Zhou and
                  Liwang Zhu and
                  Yihong Zhuang},
  title        = {{LLaDA2.0}: Scaling Up Diffusion Language Models to {100B}},
  journal      = {CoRR},
  volume       = {abs/2512.15745},
  year         = {2025},
  url          = {https://arxiv.org/abs/2512.15745},
  eprinttype   = {arXiv},
  eprint       = {2512.15745},
  bibsource    = {dblp computer science bibliography, https://dblp.org}
}

@inproceedings{DBLP:journals/corr/abs-2506-20639,
  author       = {Shansan Gong and
                  Ruixiang Zhang and
                  Huangjie Zheng and
                  Jiatao Gu and
                  Navdeep Jaitly and
                  Lingpeng Kong and
                  Yizhe Zhang},
  title        = {{DiffuCoder}: Understanding and Improving Masked Diffusion Models for
                  Code Generation},
  booktitle    = {The Fourteenth International Conference on Learning Representations, {ICLR}
                  2026},
  publisher    = {OpenReview.net},
  year         = {2026},
  url          = {https://openreview.net/forum?id=58NA3unZj5}
}

@inproceedings{DBLP:conf/icml/SahooDGWCK25,
  author       = {Subham Sekhar Sahoo and
                  Justin Deschenaux and
                  Aaron Gokaslan and
                  Guanghan Wang and
                  Justin T. Chiu and
                  Volodymyr Kuleshov},
  title        = {The Diffusion Duality},
  booktitle    = {Forty-second International Conference on Machine Learning, {ICML} 2025},
  series       = {Proceedings of Machine Learning Research},
  volume       = {267},
  pages        = {52584--52619},
  publisher    = {{PMLR}},
  year         = {2025},
  url          = {https://proceedings.mlr.press/v267/sahoo25a.html},
  bibsource    = {dblp computer science bibliography, https://dblp.org}
}

@inproceedings{DBLP:journals/corr/abs-2602-15014,
  author       = {Subham Sekhar Sahoo and
                  Jean{-}Marie Lemercier and
                  Zhihan Yang and
                  Justin Deschenaux and
                  Jingyu Liu and
                  John Thickstun and
                  Ante Juki{\'c}},
  title        = {Scaling Beyond Masked Diffusion Language Models},
  booktitle    = {Forty-third International Conference on Machine Learning, {ICML} 2026},
  series       = {Proceedings of Machine Learning Research},
  volume       = {306},
  pages        = {106997--107014},
  publisher    = {{PMLR}},
  year         = {2026},
  url          = {https://proceedings.mlr.press/v306/sahoo26a.html}
}

@article{DBLP:journals/corr/abs-2606-19005,
  author       = {Mengyu Ye and
                  Keito Kudo and
                  Wataru Ikeda and
                  Ryosuke Matsuda and
                  Keisuke Sakaguchi and
                  Jun Suzuki},
  title        = {Sumi: Open Uniform Diffusion Language Model from Scratch},
  journal      = {CoRR},
  volume       = {abs/2606.19005},
  year         = {2026},
  url          = {https://arxiv.org/abs/2606.19005},
  eprinttype   = {arXiv},
  eprint       = {2606.19005},
  bibsource    = {dblp computer science bibliography, https://dblp.org}
}

@inproceedings{DBLP:conf/iclr/ArriolaGCYQHSK25,
  author       = {Marianne Arriola and
                  Aaron Gokaslan and
                  Justin T. Chiu and
                  Zhihan Yang and
                  Zhixuan Qi and
                  Jiaqi Han and
                  Subham Sekhar Sahoo and
                  Volodymyr Kuleshov},
  title        = {Block Diffusion: Interpolating Between Autoregressive and Diffusion
                  Language Models},
  booktitle    = {The Thirteenth International Conference on Learning Representations, {ICLR}
                  2025},
  publisher    = {OpenReview.net},
  year         = {2025},
  url          = {https://openreview.net/forum?id=tyEyYT267x},
  bibsource    = {dblp computer science bibliography, https://dblp.org}
}

@inproceedings{DBLP:conf/nips/WangSSK25,
  author       = {Guanghan Wang and
                  Yair Schiff and
                  Subham Sekhar Sahoo and
                  Volodymyr Kuleshov},
  title        = {Remasking Discrete Diffusion Models with Inference-Time Scaling},
  booktitle    = {Advances in Neural Information Processing Systems 38: Annual Conference on
                  Neural Information Processing Systems 2025, NeurIPS 2025},
  year         = {2025},
  url          = {http://papers.nips.cc/paper\_files/paper/2025/hash/d877ea0dd78bdbe54830670618c1de09-Abstract-Conference.html},
  bibsource    = {dblp computer science bibliography, https://dblp.org}
}

@inproceedings{DBLP:journals/corr/abs-2505-22618,
  author       = {Chengyue Wu and
                  Hao Zhang and
                  Shuchen Xue and
                  Zhijian Liu and
                  Shizhe Diao and
                  Ligeng Zhu and
                  Ping Luo and
                  Song Han and
                  Enze Xie},
  title        = {{Fast-dLLM}: Training-free Acceleration of Diffusion {LLM} by Enabling
                  {KV} Cache and Parallel Decoding},
  booktitle    = {The Fourteenth International Conference on Learning Representations, {ICLR}
                  2026},
  publisher    = {OpenReview.net},
  year         = {2026},
  url          = {https://openreview.net/forum?id=3Z3Is6hnOT}
}

@inproceedings{DBLP:conf/nips/Ben-HamuGSNK25,
  author       = {Heli Ben{-}Hamu and
                  Itai Gat and
                  Daniel Severo and
                  Niklas Nolte and
                  Brian Karrer},
  title        = {Accelerated Sampling from Masked Diffusion Models via Entropy Bounded
                  Unmasking},
  booktitle    = {Advances in Neural Information Processing Systems 38: Annual Conference on
                  Neural Information Processing Systems 2025, NeurIPS 2025},
  year         = {2025},
  url          = {http://papers.nips.cc/paper\_files/paper/2025/hash/510e22f4a2da5212a64f1591736e2eaf-Abstract-Conference.html},
  bibsource    = {dblp computer science bibliography, https://dblp.org}
}

@inproceedings{DBLP:conf/nips/YangTLZSTW25,
  author       = {Ling Yang and
                  Ye Tian and
                  Bowen Li and
                  Xinchen Zhang and
                  Ke Shen and
                  Yunhai Tong and
                  Mengdi Wang},
  title        = {{MMaDA}: Multimodal Large Diffusion Language Models},
  booktitle    = {Advances in Neural Information Processing Systems 38: Annual Conference on
                  Neural Information Processing Systems 2025, NeurIPS 2025},
  year         = {2025},
  url          = {http://papers.nips.cc/paper\_files/paper/2025/hash/caa934a507a952698d54efb24845fc4b-Abstract-Conference.html},
  bibsource    = {dblp computer science bibliography, https://dblp.org}
}

@inproceedings{eijkelboom2024variational,
  author       = {Floor Eijkelboom and
                  Grigory Bartosh and
                  Christian A. Naesseth and
                  Max Welling and
                  Jan{-}Willem van de Meent},
  title        = {Variational Flow Matching for Graph Generation},
  booktitle    = {Advances in Neural Information Processing Systems 37: Annual Conference on
                  Neural Information Processing Systems 2024, NeurIPS 2024},
  year         = {2024},
  url          = {https://proceedings.neurips.cc/paper_files/paper/2024/hash/15b780350b302a1bf9a3bd273f5c15a4-Abstract-Conference.html}
}

@inproceedings{peebles2023dit,
  author       = {William Peebles and
                  Saining Xie},
  title        = {Scalable Diffusion Models with Transformers},
  booktitle    = {{IEEE/CVF} International Conference on Computer Vision, {ICCV} 2023},
  pages        = {4172--4182},
  publisher    = {{IEEE}},
  year         = {2023},
  doi          = {10.1109/ICCV51070.2023.00387}
}

@inproceedings{zheng2026rae,
  author       = {Boyang Zheng and
                  Nanye Ma and
                  Shengbang Tong and
                  Saining Xie},
  title        = {Diffusion Transformers with Representation Autoencoders},
  booktitle    = {The Fourteenth International Conference on Learning Representations, {ICLR}
                  2026},
  publisher    = {OpenReview.net},
  year         = {2026},
  url          = {https://openreview.net/forum?id=0u1LigJaab}
}

@inproceedings{li2026jit,
  author       = {Tianhong Li and
                  Kaiming He},
  title        = {Back to Basics: Let Denoising Generative Models Denoise},
  booktitle    = {{IEEE/CVF} Conference on Computer Vision and Pattern Recognition, {CVPR}
                  2026},
  pages        = {36115--36125},
  publisher    = {{IEEE}},
  year         = {2026},
  url          = {https://openaccess.thecvf.com/content/CVPR2026/html/Li_Back_to_Basics_Let_Denoising_Generative_Models_Denoise_CVPR_2026_paper.html}
}
\bibliographystyle{unsrtnat}

\clearpage
\appendix

\section*{Appendix Contents}
\startcontents[appendix]
\begingroup
\definecolor{tocgray}{gray}{0.40}
\hypersetup{linkcolor=.}
\setlength{\parskip}{0pt}
\titlecontents{section}[1.7em]{\addvspace{0.8em}\bfseries}{\contentslabel{1.7em}}{\hspace*{-1.7em}}{\hfill\contentspage}
\titlecontents{subsection}[4.2em]{\addvspace{0.3em}}{\contentslabel{2.5em}}{\hspace*{-2.5em}}{\,{\color{tocgray}\titlerule*[0.7pc]{.}}\contentspage}
\titlecontents*{paragraph}[4.2em]{\footnotesize\color{tocgray}\rightskip=1.6em plus 6em\relax\hyphenpenalty=10000 \exhyphenpenalty=10000 }{}{}{}[\enspace]
\edef\savedtocdepth{\the\value{tocdepth}}%
\printcontents[appendix]{}{1}{\setcounter{tocdepth}{4}}
\setcounter{tocdepth}{\savedtocdepth}
\par\vspace{0.9em}\noindent{\color{tocgray!45}\rule{\linewidth}{0.4pt}}\par
\endgroup

\section{Theory and Proofs}
\label{app:theory}

\paragraph{Indexed Notation.}
Let $\mathcal V=\{1,\ldots,V\}$ be a vocabulary of size $V$, $L$ the sequence length, and $d$ the flow embedding dimension. A token sequence $\vw=[w^{(1)},\ldots,w^{(L)}]\in\mathcal V^L$ is drawn from the data distribution $\mathcal D$. The embedding map $E:\mathcal V\to\R^d$ acts positionwise to give the clean state $X_1=E(\vw)\in\R^{d\times L}$. For a state $Z\in\R^{d\times L}$, its $\ell$th column is $Z^{(\ell)}\in\R^d$, where $\ell\in\{1,\ldots,L\}$; thus $X_1^{(\ell)}=E(w^{(\ell)})$. Parenthesized superscripts index sequence positions for both tokens and states. Capital letters denote full-sequence states: $X_t=tX_1+(1-t)X_0$, $t\in[0,1]$, is the training interpolation, $Z_t$ a sampling trajectory, and $Z$ a generic state at which a posterior or vector field is evaluated. Lowercase $z\in\R^d$ denotes a single-position state. The noise columns $X_0^{(\ell)}$ are independent and identically distributed with density $q_0$ on $\R^d$, and independent of $\vw$; we also write $q_0$ for their joint density on $\R^{d\times L}$. Our experiments use the standard Gaussian $q_0=\mathcal N(0,I_d)$. State norms are Frobenius norms, and single-position embedding norms are Euclidean norms.

\paragraph{Sequence Laws and Token Posteriors.}
We overload $q_1$ to denote both the data distribution on token sequences, $q_1(\vw)=\mathcal D(\vw)$, and the induced distribution of the embedded clean state $X_1=E(\vw)$. Let $r_1(\vw)$ be a reference sequence distribution. The shared corruption likelihood and its positionwise factors are
\[
q_{t\mid1}(Z\mid\vw)
=\prod_{\ell=1}^L q_{t\mid1}^{(\ell)}(Z^{(\ell)}\mid w^{(\ell)}),
\qquad
q_{t\mid1}^{(\ell)}(z\mid w)
=\frac{1}{(1-t)^d}\,q_0\!\left(\frac{z-tE(w)}{1-t}\right),\quad t<1,
\]
since $X_t^{(\ell)}=tE(w^{(\ell)})+(1-t)X_0^{(\ell)}$ given the clean token.
For $a\in\{q,r\}$, define the noisy-state density, full-sequence posterior, and positionwise token posterior by
\[
\begin{aligned}
a_t(Z)&\coloneqq\sum_{\vu\in\mathcal V^L}a_1(\vu)q_{t\mid1}(Z\mid\vu),\\
a_t(\vw\mid Z)&\coloneqq\frac{a_1(\vw)q_{t\mid1}(Z\mid\vw)}{a_t(Z)},\\
a_t^{(\ell)}(w\mid Z)&\coloneqq\sum_{\vu:\,u^{(\ell)}=w}a_t(\vu\mid Z).
\end{aligned}
\]
In particular, $q_t^{(\ell)}(w\mid Z)=\Pr(w^{(\ell)}=w\mid X_t=Z)$ under the data distribution, and $r_t^{(\ell)}$ is the same posterior under the reference distribution with the same corruption. The arguments distinguish noisy-state densities $a_t(Z)$ from token posteriors $a_t^{(\ell)}(w\mid Z)$. When the position index is suppressed, $q_t(w\mid Z_t)$ and $r_t(w\mid Z_t)$ denote the corresponding positionwise posteriors evaluated at $Z=Z_t$. The subscript $t\mid1$ is reserved for the corruption likelihood. Unless otherwise specified, sums over $\vu$ run over $\mathcal V^L$, and sums over individual tokens run over $\mathcal V$.

\paragraph{Model Outputs and Self-Conditioning.}
A model with parameters $\theta$ predicts a vocabulary distribution $p_\theta^{(\ell)}(\cdot\mid Z,t;\mC)$ at each position, conditioning on the full state $Z$ and, when used, the self-conditioning input $\mC\in\R^{d\times L}$. With singular logits,
\[
\begin{aligned}
p_\theta^{(\ell)}(\cdot\mid Z,t;\mC)
&=\operatorname{softmax}\!\left(\log r_t^{(\ell)}(\cdot\mid Z)+f_\theta^{(\ell)}(\cdot;Z,t;\mC)\right),\\
\widehat{X}_1^{(\ell)}(Z,t;\mC)
&\coloneqq\sum_{w\in\mathcal V}E(w)p_\theta^{(\ell)}(w\mid Z,t;\mC).
\end{aligned}
\]
Here $f_\theta^{(\ell)}$ is the learned residual logit, $\log r_t^{(\ell)}+f_\theta^{(\ell)}$ are the full logits $l_\theta$ of Section~\ref{sec:self_conditioning}, and $\widehat{X}_1^{(\ell)}$ is the $\ell$th column of the clean prediction $\widehat{X}_1$. With distilled self-conditioning guidance, the network also receives the guidance scale $\omega$ (Appendix~\ref{app:sc_training_guidance}). We omit $\mC$ and $\omega$ when they are unused or implicit.

\subsection{From Velocity to Conditional Marginals}
\label{app:velocity_marginals}

The optimal rectified-flow velocity is the conditional mean of the interpolation velocity \citep{liu2022flow, lipman2022flow}:
\[
\vv^*(Z,t)=\E[X_1-X_0\mid X_t=Z].
\]
Since $X_1-X_0=(X_1-X_t)/(1-t)$ for $t<1$, we obtain
\begin{equation}\label{eq:v_decompose}
\begin{aligned}
    \vv^*(Z,t)
    &=\E\!\left[\frac{X_1-X_t}{1-t}\,\middle|\,X_t=Z\right]\\
    &=\frac{\E[X_1\mid X_t=Z]-Z}{1-t}.
\end{aligned}
\end{equation}
The $\ell$th column of the exact clean conditional mean is
\[
\E[X_1^{(\ell)}\mid X_t=Z]
=\sum_{w\in\mathcal V}E(w)q_t^{(\ell)}(w\mid Z).
\]
Thus the positionwise token posteriors suffice to determine every column of the velocity, even when the full-sequence posterior does not factorize. Replacing $q_t^{(\ell)}$ by $p_\theta^{(\ell)}$ gives the modeled clean prediction $\widehat{X}_1$ and velocity
\[
\vv_\theta(Z,t;\mC)
=\frac{\widehat{X}_1(Z,t;\mC)-Z}{1-t}.
\]
The token-averaged cross-entropy objective is
\[
\Ls(\theta)=\E_{\vw\sim\mathcal D,\,X_0\sim q_0,\,t\sim\mathcal T}\!\left[
-\frac1L\sum_{\ell=1}^L\log p_\theta^{(\ell)}(w^{(\ell)}\mid X_t,t)\right].
\]
The expectation averages over data sequences, independent noise, and the training-time distribution. Conditioned on $X_t=Z$, the expected loss at position $\ell$ is the cross-entropy from $q_t^{(\ell)}(\cdot\mid Z)$ to $p_\theta^{(\ell)}(\cdot\mid Z,t)$, which is minimized at $p_\theta^{(\ell)}=q_t^{(\ell)}$; there $\widehat X_1=\E[X_1\mid X_t=Z]$ and $\vv_\theta=\vv^*$. With self-conditioning, each training sequence uses $\mC=\operatorname{sg}[\widehat{X}_1(X_t,t;\mathbf0)]$ with probability $1/2$ and $\mC=\mathbf0$ otherwise, where $\operatorname{sg}$ blocks gradients through the preliminary prediction. The model probabilities then also condition on $\mC$, the expectation includes this choice, and the target at each position remains $w^{(\ell)}$. Since $\mC$ is a function of $X_t$ and a draw independent of $\vw$, the minimizer is unchanged.

\subsection{Singular Logits}
\label{app:singular_logits}

\begin{theorem}
Suppose $r_1(\vw)>0$ and $\kappa^{-1}\leq q_1(\vw)/r_1(\vw)\leq\kappa$ for every $\vw\in\mathcal V^L$ and some $1\leq\kappa<\infty$. For each position $\ell$, the singular logits are
\[
\log r_t^{(\ell)}(w\mid Z)
=\log\frac{\sum_{\vu:\,u^{(\ell)}=w}r_1(\vu)q_{t\mid1}(Z\mid\vu)}
{\sum_{\vu}r_1(\vu)q_{t\mid1}(Z\mid\vu)}.
\]
Define the exact residual that the neural network $f_\theta^{(\ell)}$ approximates by
\[
f_*^{(\ell)}(w;Z,t)
\coloneqq\log q_t^{(\ell)}(w\mid Z)-\log r_t^{(\ell)}(w\mid Z).
\]
This residual is uniformly bounded: $\abs{f_*^{(\ell)}(w;Z,t)}\leq2\log\kappa$ for every position $\ell$, state $Z\in\R^{d\times L}$, token $w\in\mathcal V$, and time $t<1$.
\end{theorem}

\begin{proof}
By Bayes' rule, the positionwise data posterior is
\begin{align*}
q_t^{(\ell)}(w\mid Z)
&=\sum_{\vu:\,u^{(\ell)}=w}q_t(\vu\mid Z) \\
&=\frac{\sum_{\vu:\,u^{(\ell)}=w}q_1(\vu)q_{t\mid1}(Z\mid\vu)}
{\sum_{\vu}q_1(\vu)q_{t\mid1}(Z\mid\vu)}.
\end{align*}
With the analogous expression for $r_t^{(\ell)}$, the ratio of the two posteriors is
\begin{align*}
\frac{q_t^{(\ell)}(w\mid Z)}{r_t^{(\ell)}(w\mid Z)}
&=\frac{\sum_{\vu:\,u^{(\ell)}=w}q_1(\vu)q_{t\mid1}(Z\mid\vu)}
{\sum_{\vu:\,u^{(\ell)}=w}r_1(\vu)q_{t\mid1}(Z\mid\vu)} \\
&\quad\cdot\frac{\sum_{\vu}r_1(\vu)q_{t\mid1}(Z\mid\vu)}
{\sum_{\vu}q_1(\vu)q_{t\mid1}(Z\mid\vu)}.
\end{align*}
Since $\kappa^{-1}r_1(\vu)\leq q_1(\vu)\leq\kappa r_1(\vu)$, summing over sequences with $u^{(\ell)}=w$ yields
\begin{align*}
\kappa^{-1}\!\sum_{\vu:\,u^{(\ell)}=w}r_1(\vu)q_{t\mid1}(Z\mid\vu)
&\leq\sum_{\vu:\,u^{(\ell)}=w}q_1(\vu)q_{t\mid1}(Z\mid\vu) \\
&\leq\kappa\!\sum_{\vu:\,u^{(\ell)}=w}r_1(\vu)q_{t\mid1}(Z\mid\vu).
\end{align*}
Thus the first factor lies in $[\kappa^{-1},\kappa]$. Summing over all sequences gives the same bounds for the second factor. Therefore,
\[
\kappa^{-2}\leq
\frac{q_t^{(\ell)}(w\mid Z)}{r_t^{(\ell)}(w\mid Z)}
\leq\kappa^2,
\qquad
\abs{f_*^{(\ell)}(w;Z,t)}\leq2\log\kappa.
\]
\end{proof}

\begin{lemma}
Suppose the reference distribution factorizes as $r_1(\vw)=\prod_{\ell=1}^L r_1^{(\ell)}(w^{(\ell)})$, where $r_1^{(\ell)}$ is the reference token distribution at position $\ell$. Then the singular logits can also be computed positionwise as
\[
\log r_t^{(\ell)}(w\mid Z)
=\log\frac{r_1^{(\ell)}(w)q_{t\mid1}^{(\ell)}(Z^{(\ell)}\mid w)}
{\sum_{u\in\mathcal V}r_1^{(\ell)}(u)q_{t\mid1}^{(\ell)}(Z^{(\ell)}\mid u)}.
\]
For a uniform reference, this becomes
\[
\log r_t^{(\ell)}(w\mid Z)
=\log q_{t\mid1}^{(\ell)}(Z^{(\ell)}\mid w)
-\log\sum_{u\in\mathcal V}q_{t\mid1}^{(\ell)}(Z^{(\ell)}\mid u).
\]
\end{lemma}

\begin{proof}
Using the factorization of the reference and corruption, the numerator of the reference token posterior is
\begin{align*}
&\sum_{\vu:\,u^{(\ell)}=w}r_1(\vu)q_{t\mid1}(Z\mid\vu) \\
&\quad=\sum_{\vu:\,u^{(\ell)}=w}
  \prod_{m=1}^L r_1^{(m)}(u^{(m)})q_{t\mid1}^{(m)}(Z^{(m)}\mid u^{(m)}) \\
&\quad=r_1^{(\ell)}(w)q_{t\mid1}^{(\ell)}(Z^{(\ell)}\mid w)
  \prod_{m\neq\ell}\left[\sum_{u\in\mathcal V}r_1^{(m)}(u)q_{t\mid1}^{(m)}(Z^{(m)}\mid u)\right].
\end{align*}
Similarly, the denominator is
\begin{align*}
\sum_{\vu}r_1(\vu)q_{t\mid1}(Z\mid\vu)
&=\prod_{m=1}^L\left[\sum_{u\in\mathcal V}r_1^{(m)}(u)q_{t\mid1}^{(m)}(Z^{(m)}\mid u)\right].
\end{align*}
Canceling the factors at positions $m\neq\ell$ and taking the logarithm gives the first expression. For a uniform reference, $r_1^{(\ell)}(u)=1/V$ for all positions and tokens, and this factor also cancels.
\end{proof}

For standard Gaussian noise, $q_{t\mid1}^{(\ell)}(z\mid w)\propto\exp\!\bigl(-\norm{z-tE(w)}^2/(2(1-t)^2)\bigr)$ with a constant shared across tokens, which cancels in the first expression and yields Eq.~\eqref{eq:singular_logits_gaussian}. Our experiments use the unigram reference $r_1^{(\ell)}(w)=\pi(w)$ at every position, where $\pi(w)$ is the corpus frequency of token $w$. Only the reference posterior reduces to a function of $Z^{(\ell)}$; the data posterior and the learned residual still depend on the full state $Z$.

\subsection{From Token Commitment to Endpoint Convergence}
\label{app:endpoint_convergence}

We first show that a bounded residual preserves the commitment of the reference posterior, and then that commitment of the modeled probabilities implies convergence of the state. Fix a position $\ell$ and, along a trajectory $Z_t$, write
\[
p_t^{(\ell)}(w)=p_\theta^{(\ell)}(w\mid Z_t,t),
\quad
r_t^{(\ell)}(w)=r_t^{(\ell)}(w\mid Z_t),
\quad
f_t^{(\ell)}(w)=f_\theta^{(\ell)}(w;Z_t,t).
\]
Let $w^{(\ell)\star}\in\mathcal V$ be the token selected in the limit at position $\ell$. Suppose $r_t^{(\ell)}(w^{(\ell)\star})\to1$ and $|f_t^{(\ell)}(w)|\leq M$ uniformly in $w$ and $t<1$. Since $p_t^{(\ell)}(w)\propto r_t^{(\ell)}(w)e^{f_t^{(\ell)}(w)}$,
\begin{align*}
1-p_t^{(\ell)}(w^{(\ell)\star})
&=\frac{\sum_{w\neq w^{(\ell)\star}}r_t^{(\ell)}(w)e^{f_t^{(\ell)}(w)}}
{\sum_{u\in\mathcal V}r_t^{(\ell)}(u)e^{f_t^{(\ell)}(u)}} \\
&\leq\frac{e^M\sum_{w\neq w^{(\ell)\star}}r_t^{(\ell)}(w)}
{e^{-M}\sum_{u\in\mathcal V}r_t^{(\ell)}(u)} \\
&=e^{2M}\bigl(1-r_t^{(\ell)}(w^{(\ell)\star})\bigr)\longrightarrow0.
\end{align*}
Thus, with any uniformly bounded residual, the modeled probabilities converge to the same one-hot vector as the reference posterior.

\begin{proposition}
Consider the flow ODE $\dot Z_t=\vv_\theta(Z_t,t)$, with $\vv_\theta$ computed from the modeled probabilities $p_\theta$. If each position's probability vector $p_t^{(\ell)}$ converges to a one-hot vector selecting $w^{(\ell)\star}$ as $t\uparrow1$, then $Z_t^{(\ell)}\to E(w^{(\ell)\star})$.
\end{proposition}

\begin{proof}
Let $\widehat{X}_{1,t}=\widehat{X}_1(Z_t,t)$ denote the modeled clean prediction along the trajectory. By the assumed probability convergence,
\[
\widehat{X}_{1,t}^{(\ell)}
=\sum_{w\in\mathcal V}E(w)p_t^{(\ell)}(w)
\longrightarrow E(w^{(\ell)\star}).
\]
Using $\dot Z_t^{(\ell)}=(\widehat{X}_{1,t}^{(\ell)}-Z_t^{(\ell)})/(1-t)$, we obtain
\begin{align*}
\frac{\mathrm d}{\mathrm dt}\left(\frac{Z_t^{(\ell)}}{1-t}\right)
&=\frac{(1-t)\dot Z_t^{(\ell)}+Z_t^{(\ell)}}{(1-t)^2} \\
&=\frac{\widehat{X}_{1,t}^{(\ell)}}{(1-t)^2}.
\end{align*}
Integrating from $0$ to $t$ gives
\[
\frac{Z_t^{(\ell)}}{1-t}
=Z_0^{(\ell)}+\int_0^t\frac{\widehat{X}_{1,u}^{(\ell)}}{(1-u)^2}\,\mathrm du.
\]
Since $(1-t)^{-1}\to\infty$, L'H\^opital's rule applied coordinatewise gives
\begin{align*}
\lim_{t\uparrow1}Z_t^{(\ell)}
&=\lim_{t\uparrow1}
  \frac{Z_0^{(\ell)}+\int_0^t \widehat{X}_{1,u}^{(\ell)}(1-u)^{-2}\,\mathrm du}{(1-t)^{-1}} \\
&=\lim_{t\uparrow1}\frac{\widehat{X}_{1,t}^{(\ell)}(1-t)^{-2}}{(1-t)^{-2}} \\
&=\lim_{t\uparrow1}\widehat{X}_{1,t}^{(\ell)}=E(w^{(\ell)\star}).
\end{align*}
\end{proof}

\subsection{Hellinger--SVD Embeddings}
\label{app:hellinger_svd}

\paragraph{Hellinger Geometry.}
Let $a,b\in\mathcal V$ index tokens and $c=(j,h)\in\{-1,+1\}\times\mathcal V$ index a neighboring context, where $h$ is the preceding token for $j=-1$ and the following token for $j=+1$. Let $N_{ac}$ count occurrences of token $a$ with context $c$, and let $N_a=\sum_cN_{ac}$. The context marginal $q_C(c)\propto\sum_aN_{ac}+\eta$ is normalized over contexts, with $\eta=10^{-2}$ ensuring full support. Each token's smoothed context distribution is $Q_a(c)=(N_{ac}+\tau q_C(c))/(N_a+\tau)$; smoothing toward the marginal stabilizes estimates for rare tokens, and held-out log loss of context prediction selects $\tau=2000$ (Appendix~\ref{app:embedding_details}). It is normalized:
\[
\sum_c Q_a(c)
=\frac{\sum_c N_{ac}+\tau\sum_c q_C(c)}{N_a+\tau}
=1.
\]
The squared Hellinger distance is
\begin{align*}
d_{\mathrm{Hel}}^2(Q_a,Q_b)
&\coloneqq\frac12\sum_c\left(\sqrt{Q_a(c)}-\sqrt{Q_b(c)}\right)^2 \\
&=1-\sum_c\sqrt{Q_a(c)Q_b(c)}.
\end{align*}
Therefore, with $H_{ac}=\sqrt{Q_a(c)}$,
\[
\norm{H_{a,:}-H_{b,:}}^2=2d_{\mathrm{Hel}}^2(Q_a,Q_b).
\]

\paragraph{Weighted Low-Dimensional Projection.}
Let $\rho_a=N_a/\sum_bN_b$ be token frequencies, $\bar H=\sum_a\rho_aH_{a,:}$ the frequency-weighted mean row, and $\mathbf1\in\R^V$ the all-ones column. Define the centered matrix $A=H-\mathbf1\bar H$ and its weighted version $B=WA$ with $W=\operatorname{diag}(\sqrt\rho)$. Let $P_d\in\R^{2V\times d}$ contain the leading $d$ right singular vectors of $B$. The projected embeddings are
\[
\widetilde E(a)=A_{a,:}P_d=(H_{a,:}-\bar H)P_d.
\]
The span of $P_d$ minimizes the frequency-weighted reconstruction error $\sum_a\rho_a\norm{A_{a,:}(I-P_dP_d^\top)}^2$ over $d$-dimensional subspaces; we compute it with randomized SVD. Since $P_d$ has orthonormal columns, projection changes the Hellinger geometry only through the discarded component:
\[
\norm{\widetilde E(a)-\widetilde E(b)}^2
=2d_{\mathrm{Hel}}^2(Q_a,Q_b)
-\norm{(H_{a,:}-H_{b,:})(I-P_dP_d^\top)}^2.
\]

\paragraph{Sparse Computation.}
Define
\[
\lambda_a=\sqrt{\frac{\tau}{N_a+\tau}},
\qquad
\nu_c=\sqrt{q_C(c)},
\qquad
S_{ac}=\frac{\sqrt{N_{ac}+\tau q_C(c)}-\sqrt{\tau q_C(c)}}{\sqrt{N_a+\tau}}.
\]
Then
\begin{align*}
H_{ac}
&=\frac{\sqrt{N_{ac}+\tau q_C(c)}}{\sqrt{N_a+\tau}} \\
&=S_{ac}+\frac{\sqrt{\tau q_C(c)}}{\sqrt{N_a+\tau}}
=S_{ac}+\lambda_a\nu_c,
\end{align*}
so $H=S+\boldsymbol\lambda\boldsymbol\nu^\top$. Since $S_{ac}=0$ whenever $N_{ac}=0$, $S$ has the sparsity pattern of the count matrix. The weighted mean is
\[
\bar H=\rho^\top S+(\rho^\top\boldsymbol\lambda)\boldsymbol\nu^\top.
\]
For $\mathbf x\in\R^{2V}$ and $\mathbf y\in\R^V$, the matrix products needed for randomized SVD are
\begin{align*}
B\mathbf x
&=W\left[S\mathbf x+\boldsymbol\lambda(\boldsymbol\nu^\top\mathbf x)
  -\mathbf1(\bar H\mathbf x)\right], \\
B^\top\mathbf y
&=S^\top W\mathbf y+\boldsymbol\nu(\boldsymbol\lambda^\top W\mathbf y)
  -\bar H^\top(\mathbf1^\top W\mathbf y).
\end{align*}
Finally,
\[
\widetilde E
=SP_d+\boldsymbol\lambda(\boldsymbol\nu^\top P_d)-\mathbf1(\bar H P_d).
\]
These expressions apply centering and frequency weighting without materializing the dense vocabulary-by-context matrix.

\paragraph{Final Normalization.}
Let $E^\dagger(a)$ denote the projected rows $\widetilde E(a)$ after perturbing exact duplicates. With the corpus unigram distribution $\pi$, define
\[
\mu_\pi=\sum_a\pi(a) E^\dagger(a),
\qquad
\sigma_\pi^2=\sum_a\pi(a)\norm{E^\dagger(a)-\mu_\pi}^2,
\qquad
E(a)=\frac{\sqrt d}{\sigma_\pi}\bigl(E^\dagger(a)-\mu_\pi\bigr).
\]
Then
\begin{align*}
\sum_a\pi(a) E(a)
&=\frac{\sqrt d}{\sigma_\pi}\left(\sum_a\pi(a) E^\dagger(a)-\mu_\pi\sum_a\pi(a)\right)=0, \\
\sum_a\pi(a)\norm{E(a)}^2
&=\frac{d}{\sigma_\pi^2}\sum_a\pi(a)\norm{E^\dagger(a)-\mu_\pi}^2=d.
\end{align*}
The weights $\rho$ determine the projection subspace, while $\pi$ determines the final centering and scale.

\section{Experimental Details}
\label{app:experimental_details}

\subsection{Data Preparation and Fixed Embeddings}
\label{app:embedding_details}

\paragraph{Corpus Split and Sequence Packing.}
We use the OpenWebText training split, reserving the last 100,000 documents for validation and training on the rest. Each document is tokenized with the GPT-2 tokenizer and followed by an EOS token, and documents are concatenated within each preprocessing batch. We divide the stream into blocks of 1,022 tokens, discard the incomplete remainder, and add BOS and EOS to obtain sequences of length 1,024; GPT-2 uses the same token ID for both. Token distributions are defined over the 50,257-token GPT-2 vocabulary, excluding padding entries.

\paragraph{Offline Statistics and Smoothing Selection.}
Embedding construction uses packed training windows sampled from contiguous chunks spread across the corpus, with a quarter of the chunks held out from count estimation for offline embedding evaluation. Adjacent-token pairs are counted within packed windows, including special tokens. The smoothing strength of the two-sided context distribution $Q_a(c)$ of Appendix~\ref{app:hellinger_svd} is selected by held-out log loss on validation windows, giving $\tau=2000$.

\paragraph{Default Hellinger--SVD Embedding.}
The default embedding follows Appendix~\ref{app:hellinger_svd}: we retain 32 components of a randomized SVD, perturb exactly duplicated rows with small Gaussian noise, and apply the unigram-weighted centering and global rescaling described there. The embedding is then frozen. Table~\ref{tab:embedding_configs} lists the dimension and endpoint $t_{\mathrm{end}}$ of each reported embedding.

\begin{table}[htbp]
    \centering
    \caption{\textbf{Fixed embedding configurations.} Each representation has unigram-weighted mean squared norm 32 and its own precomputed training-time distribution. $t_{\mathrm{end}}$ is the endpoint of the training-time distribution.}
    \label{tab:embedding_configs}
    \begin{tabular}{llrr}
        \toprule
        Embedding & Reported configuration & Dimension & $t_{\mathrm{end}}$ \\
        \midrule
        Hellinger--SVD & Default & 32 & 0.943 \\
        Gaussian & Small embedding ablation & 32 & 0.620 \\
        Scaled one-hot & Small embedding ablation & 50,257 & 0.557 \\
        \bottomrule
    \end{tabular}
\end{table}

\paragraph{Gaussian and One-Hot Controls.}
For the Gaussian control, we draw independent $\mathcal N(0,I_{32})$ rows, subtract their unigram-weighted mean, and rescale them to unigram-weighted mean squared norm 32. For the one-hot control, we use the uncentered embeddings $E(a)=\sqrt{32}\,e_a$, where $e_a$ is the vocabulary basis vector, matching the mean squared norm of the 32-dimensional representations. The one-hot input projection increases the model to approximately 215M parameters.

\paragraph{Precomputation of Training-Time Distributions.}
\label{app:training_time_schedule}
We build the information-arc-length schedule separately for each fixed representation at its final scale. The calculation involves a single position, so we suppress $(\ell)$ and write $X_0\in\R^d$ for the noise at that position. With $\gamma=t/(1-t)$, the rescaled Gaussian channel is $Y_\gamma=X_t/(1-t)=\gamma E(w)+X_0$, where $X_0\sim\mathcal N(0,I_d)$. For a token prior $\nu$, the analytic posterior is
\begin{equation}
    r_{\gamma,\nu}(w\mid Y_\gamma)
    =\frac{\nu(w)\exp\!\left(-\tfrac12\norm{Y_\gamma-\gamma E(w)}^2\right)}
    {\sum_{u\in\mathcal V}\nu(u)\exp\!\left(-\tfrac12\norm{Y_\gamma-\gamma E(u)}^2\right)},
    \label{eq:embedding_schedule_posterior}
\end{equation}
the positionwise reference posterior at $t=\gamma/(1+\gamma)$. We use the corpus unigram prior $\pi$ and a smoothed bigram prior $\nu_h$ conditioned on the preceding clean token $h$. On held-out adjacent-token pairs $(h,w)$ and shared Gaussian noise draws, we estimate the expected posterior log-probability gains over these priors:
\begin{equation}
\begin{aligned}
    R_{\mathrm c}(\gamma)
    &=\widehat{\mathbb E}\!\left[\log\frac{r_{\gamma,\nu_h}(w\mid Y_\gamma)}{\nu_h(w)}\right], \\
    R_{\mathrm u}(\gamma)
    &=\widehat{\mathbb E}\!\left[\log\frac{r_{\gamma,\pi}(w\mid Y_\gamma)}{\pi(w)}\right].
\end{aligned}
\label{eq:embedding_schedule_rates}
\end{equation}
The conditional curve measures identification with local context, using the clean neighbor as an offline proxy for context, while the unigram curve retains the context-free identification tail.

The endpoint $\gamma_{\mathrm{end}}$ is the smallest channel level at which the unigram posterior identifies the token nearly perfectly, with top-1 error and entropy relative to the prior both at most $10^{-3}$; for the default embedding, $t_{\mathrm{end}}=\gamma_{\mathrm{end}}/(1+\gamma_{\mathrm{end}})\approx0.943$. We evaluate both gain curves on a grid $0=\gamma_0<\gamma_1<\cdots<\gamma_J=\gamma_{\mathrm{end}}$ with logarithmically spaced positive levels, rescale each to increase monotonically from $0$ to $1$ (using isotonic regression), and write $\mathbf r_j=(r_{\mathrm c,j},r_{\mathrm u,j})$ for the rescaled pair at $\gamma_j$. Their normalized cumulative arc length defines
\begin{equation}
    \begin{aligned}
    \Delta a_j&=\norm{\mathbf r_j-\mathbf r_{j-1}}_2,
    &u_j&=\frac{\sum_{k=1}^{j}\Delta a_k}{\sum_{k=1}^{J}\Delta a_k},\quad u_0=0,\\
    t&=\frac{\gamma(U)}{1+\gamma(U)},
    &U&\sim\mathcal U(0,1),
    \end{aligned}
    \label{eq:embedding_training_schedule}
\end{equation}
where $\gamma(u)$ linearly interpolates the inverse map from $u_j$ to $\gamma_j$. The schedule is computed once from corpus statistics and the fixed embedding, without evaluating the flow network. We stratify $U$ across the global training batch and apply no time-dependent loss weight.

\subsection{Model Architecture and Training}

\paragraph{Model Architecture.}
We parameterize the neural residual in \dflow with a bidirectional diffusion Transformer on sequences of 1,024 tokens, using the fixed 32-dimensional embeddings of Table~\ref{tab:embedding_configs}. At each position, we concatenate the noisy embedding with a 32-dimensional self-conditioning estimate of the clean embedding and project the result to the Transformer width. Each block combines multi-head self-attention with rotary positional embeddings (RoPE)~\citep{su2024roformer}, query--key RMS normalization, a sigmoid gate on the attention output~\citep{qiu2025gated}, and a SwiGLU feed-forward network. Both sublayers use RMS normalization with time-dependent affine modulation and residual gates, driven by a sinusoidal time embedding followed by a two-layer MLP, and dropout $0.1$ on the residual branches. A final time-modulated RMS normalization and linear projection produce vocabulary logits, to which we add the analytic singular logits with the unigram reference distribution, except in the ablation that removes this term. Table~\ref{tab:model_sizes} summarizes the configurations. Small, Medium, and Large match the hidden widths and depths of ELF-B, ELF-M, and ELF-L~\citep{elf2026}, which we accordingly call ELF-Small, ELF-Medium, and ELF-Large. We implement attention with FlashAttention~\citep{dao2022flashattention,dao2023flashattention2}; for compatibility with its backward kernels, Medium and Large use 11 and 20 heads instead of ELF's 16, while Small keeps ELF's 12.

\begin{table}[htbp]
    \centering
    \caption{\textbf{Architecture configurations of \dflow.} Parameter counts include all trainable projections, attention gates, and time-conditioning layers, but exclude the fixed token embeddings. FFN width denotes the intermediate dimension of the SwiGLU network.}
    \label{tab:model_sizes}
    \begin{tabular}{lrrrrr}
        \toprule
        Model & Parameters (M) & Layers & Hidden width & Attention heads & FFN width \\
        \midrule
        Small       & 138     & 12 & 768   & 12 & 2,048 \\
        Medium      & 421     & 24 & 1,056 & 11 & 2,816 \\
        Large       & 784     & 32 & 1,280 & 20 & 3,456 \\
        \bottomrule
    \end{tabular}
\end{table}

\paragraph{Model Training.}
\label{app:model_training}
All reported models use a global batch of 1,024 sequences with the preprocessing of Appendix~\ref{app:embedding_details}. We minimize token-averaged cross-entropy, drawing interpolation times by stratified inverse-CDF sampling from the embedding's precomputed training-time distribution on $[0,t_{\mathrm{end}}]$, with no time-dependent loss weight. For self-conditioning, a forward pass without gradient tracking estimates the clean embedding at the same noisy input and time; independently for each sequence, we keep this estimate with probability $0.5$ and otherwise replace it with zero. We optimize hidden, time-embedding, and modulation matrices with Muon (Nesterov momentum $0.95$, weight decay $0.02$), and the input and output projections, normalization gains, and biases with AdamW ($\beta_1=0.9$, $\beta_2=0.99$, no weight decay). Both use a base learning rate of $3\times10^{-4}$ with 2,500 linear warmup steps followed by a constant schedule; Muon matches its update RMS to AdamW. Training uses bfloat16 mixed precision, gradient clipping at norm $1.0$, and an EMA with decay $0.9999$ for evaluation.

\paragraph{Reflow Training.}
\label{app:reflow_training}
Couplings are generated from the EMA weights of $\theta_k$ with 256-NFE Euler sampling, at guidance scale $\omega=1.3$ in the first round and without guidance ($\omega=1$) in the second, giving about 0.5M noise--sequence couplings per round. Each round fine-tunes for 4,000 steps with batch size 512 and learning rate $10^{-4}$. Eq.~\eqref{eq:reflow_rounds} uses $\lambda=0.7$ in the first round and $\lambda=0.55$ in the second.

\paragraph{Compute Accounting.}
A training run of the default Small configuration takes approximately 300 NVIDIA GH200 GPU-hours.

\paragraph{Numerical Precision.}
The Transformer residual stream and vocabulary projection run in bfloat16, while the flow state, analytic singular logits, and posterior embedding expectations stay in FP32. Network logits are promoted to FP32 before the analytic term is added, and the softmax and probability-weighted embedding sum are computed in FP32. This limits rounding of token-logit differences as the singular term grows near the endpoint and avoids storing small integration increments in bfloat16. Sampling updates, the self-conditioning carry, and the momentum-guidance velocity average also remain in FP32. Grid nodes and interval coefficients are computed in FP64.

\paragraph{Memory-Efficient Vocabulary Operations.}
We limit the memory of vocabulary-sized tensors by processing token positions in chunks. During training, each chunk of hidden features is projected to vocabulary logits, augmented with the analytic term, and passed to a fused cross-entropy kernel. A custom autograd operation computes each chunk's hidden-feature and output-head gradients immediately and saves them for the outer backward pass, so training never materializes the full logit tensor. The self-conditioning pass likewise computes embedding expectations chunk by chunk. During sampling, the network produces the full bfloat16 logit tensor, but singular-logit addition, softmax, and the embedding expectation are performed in FP32 chunks, and final decoding applies a chunkwise argmax.

\subsection{Distilled Self-Conditioning Guidance}
\label{app:sc_training_guidance}

We write $l_\theta(X_t,t;\mC,\omega)=\log r_t(\cdot\mid X_t)+f_\theta(\cdot;X_t,t;\mC,\omega)$ for the full logits, including the analytic singular term, of a network that also receives the guidance scale $\omega$. Throughout this subsection, $w$ denotes the clean target token.

\paragraph{Explicit Self-Conditioning Guidance.}
A self-conditioned model defines two posteriors at each noisy state: a cold posterior with $\mC=\mathbf 0$ and a warm posterior conditioned on the carried clean estimate. Contrasting them in logit space gives
\begin{equation}
    l^{\mathrm g}=l(\mathbf 0)+\omega\bigl(l(\mC)-l(\mathbf 0)\bigr),
    \qquad
    p^{\mathrm g}(\cdot)\propto p(\cdot\mid\mathbf 0)^{1-\omega}\,p(\cdot\mid\mC)^{\omega},
    \label{eq:sc_cfg_explicit}
\end{equation}
where $\omega=1$ recovers unguided self-conditioning. The singular logits are identical in both passes, so guidance acts only on the network residual. The clean estimate $\widehat X_1=\sum_wE(w)\,p^{\mathrm g}(w)$ determines the velocity and is carried as the next $\mC$. Applied explicitly, Eq.~\eqref{eq:sc_cfg_explicit} requires a cold and a warm evaluation at every step after the first, nearly doubling the sampling cost.

\paragraph{Scale Conditioning.}
To obtain guided predictions from a single evaluation, the network receives $\omega$ as an input. For each training sequence, we draw $\omega$ log-uniformly on $[1.5,6]$, the same distribution as in ELF. The scale is encoded by a sinusoidal embedding and a two-layer MLP, mirroring the time embedding, and added to the time embedding that drives the adaptive normalization and residual gates. The final MLP layer is zero-initialized, so the network initially ignores $\omega$.

\paragraph{Training Objective.}
At each training step, we evaluate the online network twice without gradient tracking at $(X_t,t,\omega)$, with dropout disabled. The cold pass gives $l^{\mathrm u}=l_\theta(X_t,t;\mathbf 0,\omega)$ and the clean estimate $\widehat X_1^{\mathrm u}$; the warm pass, conditioned on $\widehat X_1^{\mathrm u}$, gives $l^{\mathrm c}=l_\theta(X_t,t;\widehat X_1^{\mathrm u},\omega)$. With a per-sequence mask $k\sim\mathrm{Bernoulli}(1/2)$ shared across positions, the gradient-carrying pass uses $\mC=k\,\widehat X_1^{\mathrm u}$, and we minimize the shifted-logit CE objective in Eq.~\eqref{eq:sc_training_guidance}.
The singular logits cancel in $\Delta$, and gradients flow only through the unshifted logits $l_\theta(X_t,t;\mC,\omega)$. Sequences with $k=0$ therefore receive the ordinary CE objective, and sequences with $k=1$ receive CE on shifted logits, so the loss remains a single token-averaged cross-entropy.

\paragraph{Stationary Point.}
Let $l^{*}$ denote the log-posterior that CE fits for the self-conditioned input without the shift. For $k=1$, CE drives $l_\theta(\cdot;\mC,\omega)-\Delta$ toward $l^{*}$ up to an additive constant. At a stationary point where the no-gradient warm pass agrees with the gradient-carrying pass, $l^{\mathrm c}=l_\theta(\cdot;\mC,\omega)$, so
\begin{equation*}
    \tfrac{1}{\omega}\,l_\theta(\cdot;\mC,\omega)+\bigl(1-\tfrac{1}{\omega}\bigr)\,l^{\mathrm u}=l^{*}
    \quad\Longrightarrow\quad
    l_\theta(\cdot;\mC,\omega)=l^{\mathrm u}+\omega\bigl(l^{*}-l^{\mathrm u}\bigr),
\end{equation*}
again up to an additive constant. A single warm evaluation at scale $\omega$ thus reproduces the guidance of Eq.~\eqref{eq:sc_cfg_explicit}.

\paragraph{Relation to ELF.}
ELF~\citep{elf2026} uses the same construction, conditioning the network on a per-example guidance scale and building the target from its own stop-gradient cold and warm predictions, but applies the shift to an MSE velocity target, $\vv+(1-1/\omega)\,\operatorname{sg}[\vv_\theta(\mC)-\vv_\theta(\mathbf 0)]$. Its guidance therefore extrapolates the clean-embedding estimate linearly, whereas ours extrapolates the token logits as in Eq.~\eqref{eq:sc_cfg_explicit} and takes the clean estimate as the mean embedding under the guided posterior $p^{\mathrm g}$; the two generally differ. ELF also trains a separate unguided CE decoding branch, while our guidance enters the only training objective.

\paragraph{Training and Inference.}
Models with distilled self-conditioning guidance are trained from scratch with the recipe of Appendix~\ref{app:model_training}. Relative to plain self-conditioning, each training step needs one additional forward pass without gradient tracking, since the cold pass also supplies the self-conditioning estimate. At inference, we feed a fixed $\omega$ to every evaluation, including the first and the final decoding evaluation, use one function evaluation per step, and carry the guided clean estimate as the self-conditioning input.

\subsection{Evaluation and Sampling}

\paragraph{Evaluation Protocol.}
For each plotted point, we generate 1,024 unconditional sequences of 1,024 tokens from the EMA parameters of the final checkpoint. We report generative perplexity and mean per-sequence unigram entropy. For generative perplexity, we decode the samples to text, retokenize them with the GPT-2 Large evaluator's tokenizer, truncate to at most 1,024 tokens, and exponentiate the mean next-token negative log-likelihood over all scored tokens, which include non-EOS targets and the first EOS in each sequence. For diversity, we compute the unigram entropy of each generated sequence in nats and average across samples.

\paragraph{Inference Settings.}
\label{app:b_sampling_schedule}
We sample with a nonuniform time grid and a time-dependent SDE noise level, tuned once on the default Small model at 64 NFE and shared across model sizes; Euler sampling uses the same grid. In the sampling-budget ablation, only the number of grid points changes. Self-conditioning starts from zero and carries the previous clean estimate between steps, so each step requires one function evaluation, and the last evaluation decodes each position by argmax. Momentum guidance~\citep{liao2026momentum} extrapolates the current velocity away from an exponential moving average of previous velocities; we fix its base decay to $\beta=0.8$ and sweep the guidance strength $\alpha\in\{0,0.1,0.2,0.3,0.5,0.7\}$, where $\alpha=0$ is the unguided baseline.

\paragraph{Finite-Time Token Decoding.}
Under the assumptions of Appendix~\ref{app:endpoint_convergence}, singular logits ensure one-hot commitment as $t\uparrow1$, but the sampler stops before $t=1$. It uses $N$ evaluation times $t_0<t_1<\cdots<t_{N-1}$; write $Z_i\in\R^{d\times L}$ for the numerical state at $t_i$ and $\mC_i$ for its self-conditioning input. The final evaluation is at $t_{N-1}<1$, and we decode each position by
\[
\widehat w^{(\ell)}=\arg\max_{w\in\mathcal V}
p_\theta^{(\ell)}(w\mid Z_{N-1},t_{N-1};\mC_{N-1}).
\]

\paragraph{Finite-Step SDE Update.}
\label{app:sde_discretization}
The clean prediction gives the Gaussian score $\boldsymbol{s}_\theta=(t\widehat{X}_1-Z_t)/(1-t)^2$, with arguments omitted. We sample using the score-corrected SDE~\citep{ma2024sit}
\begin{equation}
    \mathrm dZ_t
    = \left[\vv_\theta(Z_t,t)
      +\sigma(1-t)^2\boldsymbol{s}_\theta(Z_t,t)\right]\mathrm dt
      +\sqrt{2\sigma}\,(1-t)\,\mathrm d\mW_t,
    \qquad \sigma\geq0,
    \label{eq:tokenflow_sde}
\end{equation}
where $\mW_t$ is standard Brownian motion and $\sigma$ controls stochasticity; $\sigma=0$ recovers the flow ODE.
We discretize Eq.~\eqref{eq:tokenflow_sde} by refreshing the estimated Gaussian source component. Initialize $Z_0\sim q_0$ and $\mC_0=\mathbf0$. For $i=0,\ldots,N-2$, let $h_i=t_{i+1}-t_i$, $b_i=1-t_i$, $\widehat{X}_{1,i}=\widehat{X}_1(Z_i,t_i;\mC_i)$, and $\widehat{\boldsymbol\epsilon}_i=(Z_i-t_i\widehat{X}_{1,i})/b_i$, and let $\boldsymbol\xi_i\in\R^{d\times L}$ be fresh standard Gaussian noise. The update is
\begin{equation}
\begin{aligned}
    \delta_i &= \min\{\sigma h_i b_i,\,b_i\}, \\
    \widetilde{Z}_i &= Z_i-\delta_i\widehat{\boldsymbol\epsilon}_i
      +\sqrt{b_i^2-(b_i-\delta_i)^2}\,\boldsymbol\xi_i, \\
    Z_{i+1} &= \widetilde{Z}_i
      +\frac{h_i}{b_i}\bigl(\widehat{X}_{1,i}-\widetilde{Z}_i\bigr).
\end{aligned}
\label{eq:sde_discrete}
\end{equation}
The first update replaces part of the estimated source noise with fresh noise, and the second moves toward the same clean prediction without another network call. We set $\mC_{i+1}=\widehat{X}_{1,i}$. As $h_i\to0$, the drift and diffusion coefficients converge to those of Eq.~\eqref{eq:tokenflow_sde}, and $\sigma=0$ gives Euler integration of the flow ODE. The grid has $N$ evaluations and $N-1$ state updates; the last evaluation performs token decoding.

When momentum guidance is enabled, we compute $\vv_i=(\widehat{X}_{1,i}-Z_i)/b_i$ and form $\vv_i^{\mathrm g}=(1+\alpha)\vv_i-\alpha\vm_i$, where the velocity average $\vm_i$ is initialized by $\vm_0=\vv_0$ and updated as $\vm_{i+1}=(1-\beta^{Nh_i})\vv_i+\beta^{Nh_i}\vm_i$. We replace $\widehat{X}_{1,i}$ in Eq.~\eqref{eq:sde_discrete}, including in $\widehat{\boldsymbol\epsilon}_i$, by $Z_i+b_i\vv_i^{\mathrm g}$, while self-conditioning carries the unguided prediction $\widehat{X}_{1,i}$.

\subsection{Ablation Settings}
\label{app:ablation_details}

\paragraph{Singular Logits and Embeddings.}
The models in Figure~\ref{fig:ablations}(a) and (b) are trained for 80,000 steps without distilled self-conditioning guidance and evaluated at 64 NFE with the momentum-guidance sweep. They use a logit-normal time grid (location $-0.8$, scale $1.2$) on $[0,t_{\mathrm{end}}]$ with a constant SDE noise scale of $4$; Euler sampling uses the same grid. The Gaussian and one-hot controls are specified in Appendix~\ref{app:embedding_details}.

\paragraph{Training Objective.}
The MSE variant keeps the network, its logit output, and the clean estimate $\widehat X_1=\sum_wE(w)p_\theta(w\mid X_t,t)$ of the CE model, and trains the velocity $\vv_\theta=\tfrac{1}{1-t}(\widehat X_1-X_t)$ with the RF loss of Section~\ref{sec:rectified_flow}, averaged over positions and embedding coordinates. Interpolation times follow the same training-time distribution, with no additional weight. For distilled self-conditioning guidance, we reuse the two no-gradient passes of Eq.~\eqref{eq:sc_training_guidance} but, as in ELF, shift the velocity target by $(1-1/\omega)\,k\operatorname{sg}[\vv_\theta(\widehat X_1^{\mathrm u})-\vv_\theta(\mathbf 0)]$ rather than the logits. Both models are evaluated at 80,000 training steps with the same sampler.

\paragraph{Flow Embedding Dimension.}
Each $d$-dimensional embedding follows the construction of Appendix~\ref{app:embedding_details} with $d$ retained singular directions, and receives its own training-time distribution and endpoint $t_{\mathrm{end}}$, which decreases from $0.969$ at $d=16$ to $0.731$ at $d=256$.

\begin{table}[h]
    \centering
    \small
    \caption{\textbf{Flow embedding dimension sweep for \dflow-Small.} Selected settings from the sweep in Figure~\ref{fig:ablations}(d); SDE sampling at 64 NFE. For each $d$, the first row gives the guidance scale $\omega$ and the second the Gen.\ PPL / entropy.}
    \label{tab:flow_dim_sweep}
    \begin{tabular}{llccccc}
        \toprule
        $d$ & & \multicolumn{5}{c}{Gen.\ PPL / entropy} \\
        \midrule
        16  & $\omega$ & 1.7 & 1.8 & 1.9 & 2.0 & 2.1 \\
            & & 46.5 / 5.32 & 43.0 / 5.31 & 39.7 / 5.30 & 36.7 / 5.30 & 33.9 / 5.29 \\
        \addlinespace
        32  & $\omega$ & 1.7 & 1.8 & 1.9 & 2.0 & 2.1 \\
            & & 24.1 / 5.45 & 23.6 / 5.45 & 19.1 / 5.37 & 14.9 / 5.22 & 11.6 / 5.07 \\
        \addlinespace
        64  & $\omega$ & 2.2 & 2.4 & 2.6 & 2.8 & 3.0 \\
            & & 25.4 / 5.51 & 23.7 / 5.48 & 22.0 / 5.43 & 20.4 / 5.38 & 19.3 / 5.32 \\
        \addlinespace
        128 & $\omega$ & 2.0 & 2.2 & 2.4 & 2.6 & 2.8 \\
            & & 37.6 / 5.57 & 35.0 / 5.55 & 32.2 / 5.52 & 29.7 / 5.48 & 27.3 / 5.42 \\
        \addlinespace
        256 & $\omega$ & 1.8 & 2.0 & 2.2 & 2.4 & 2.6 \\
            & & 56.6 / 5.61 & 52.1 / 5.61 & 48.2 / 5.59 & 44.5 / 5.59 & 42.1 / 5.58 \\
        \bottomrule
    \end{tabular}
\end{table}

\paragraph{Reflow Rounds.}
Figure~\ref{fig:reflow_radar} compares the base model at 32 NFE with the model after two reflow rounds at 16 NFE, each evaluated with 1,024 samples. With half the NFE, the reflowed model reaches lower Gen.\ PPL at every model size ($26.4$ versus $35.2$ for Small, $22.2$ versus $24.5$ for Medium, and $19.4$ versus $22.6$ for Large), with entropy within $0.04$ nats of the base model.

\begin{figure}[htbp]
    \centering
    \includegraphics[width=\linewidth]{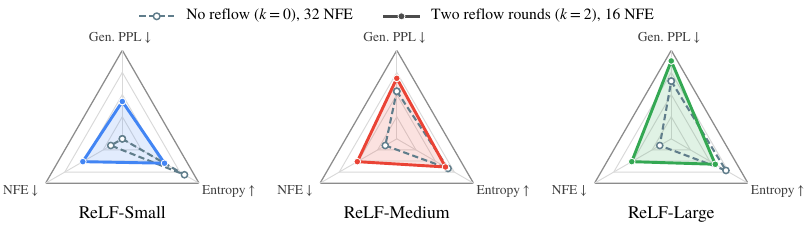}
    \caption{\textbf{Two reflow rounds at half the NFE.} Each panel compares the base model at 32 NFE (dashed) with the model after two reflow rounds at 16 NFE (solid), with vertices Gen.\ PPL, entropy, and NFE; outward is better on every axis. Gen.\ PPL is normalized over the six plotted runs, with the highest value at the center, and entropy and NFE over the reflow-round sweep ($k=0$--$3$ at 8, 16, and 32 NFE for all sizes); Gen.\ PPL and NFE use a log scale.}
    \label{fig:reflow_radar}
\end{figure}

\subsection{Quality--Diversity Frontiers and Autoregressive Baselines}
\label{app:training_guidance_pareto}

Table~\ref{tab:training_guidance_pareto} reports the quality--diversity frontiers of the Small, Medium, and Large \dflow models in Figure~\ref{fig:system_comparison}(a), all trained with distilled self-conditioning guidance. We evaluate EMA checkpoints using SDE sampling with noise scale $24$ and 64 NFE, including decoding, with 1,024 sequences per configuration. For each model size, we keep the non-dominated points: a point is dominated if another setting achieves at least as high an entropy and at most as high a Gen.\ PPL, with strict improvement in at least one. Because \dflow-Large varies little across guidance scales, only two of its settings are non-dominated; Table~\ref{tab:l_full_sweep} lists all eight.

The table also lists the six temperature settings of each AR baseline shown in Figure~\ref{fig:system_comparison}(a). AR-Small, AR-Medium, and AR-Large are causal Transformers with the hidden widths and depths of the corresponding \dflow models (Table~\ref{tab:model_sizes}) and have 124M, 374M, and 699M parameters. Although cosine learning-rate decay with the final checkpoint and no EMA is a more common recipe for AR training, our comparison found that a constant learning rate with EMA yields better generation performance. We therefore adopt the latter and evaluate EMA weights after 80,000, 200,000, and 200,000 training steps for Small, Medium, and Large, respectively, generating each length-1,024 sequence ancestrally after BOS and scoring it with the same metrics.

\begin{table}[htbp]
    \centering
    \small
    \caption{\textbf{OpenWebText Pareto results and autoregressive baselines.} \dflow uses 64 NFE; its frontiers are computed separately within each model size over $\omega\in\{1.5,1.6,1.7,1.8,1.9,2.0,2.1,2.3\}$. AR baselines use ancestral sampling over the temperatures shown in Figure~\ref{fig:system_comparison}(a). Entropy is mean per-sequence unigram entropy; Gen.\ PPL is evaluated by GPT-2 Large.}
    \label{tab:training_guidance_pareto}
    \begin{tabular}{lrrr}
        \toprule
        Model & Guidance $\omega$ & Entropy (nats) $\uparrow$ & Gen.\ PPL $\downarrow$ \\
        \midrule
        \dflow-Small & 1.7 & 5.45 & 24.1 \\
          & 1.8 & 5.45 & 23.6 \\
          & 1.9 & 5.37 & 19.1 \\
          & 2.0 & 5.22 & 14.9 \\
          & 2.1 & 5.07 & 11.6 \\
          & 2.3 & 4.94 & 9.8 \\
        \midrule
        \dflow-Medium & 1.6 & 5.48 & 19.3 \\
          & 1.7 & 5.48 & 18.6 \\
          & 1.8 & 5.47 & 17.8 \\
          & 1.9 & 5.45 & 16.8 \\
          & 2.0 & 5.41 & 15.6 \\
          & 2.1 & 5.36 & 14.3 \\
          & 2.3 & 5.31 & 13.3 \\
        \midrule
        \dflow-Large & 2.0 & 5.51 & 17.6 \\
          & 1.9 & 5.51 & 17.3 \\
        \midrule
        Model & Temperature $T$ & Entropy (nats) $\uparrow$ & Gen.\ PPL $\downarrow$ \\
        \midrule
        AR-Small & 1.00 & 5.54 & 38.1 \\
             & 0.98 & 5.46 & 30.4 \\
             & 0.96 & 5.40 & 25.0 \\
             & 0.94 & 5.32 & 20.9 \\
             & 0.92 & 5.25 & 17.4 \\
             & 0.90 & 5.20 & 15.0 \\
        \midrule
        AR-Medium & 1.00 & 5.54 & 27.9 \\
             & 0.98 & 5.46 & 22.3 \\
             & 0.96 & 5.41 & 18.5 \\
             & 0.94 & 5.36 & 15.7 \\
             & 0.92 & 5.29 & 13.4 \\
             & 0.90 & 5.24 & 11.7 \\
        \midrule
        AR-Large & 1.00 & 5.53 & 24.3 \\
             & 0.98 & 5.47 & 20.2 \\
             & 0.96 & 5.41 & 16.8 \\
             & 0.94 & 5.36 & 14.3 \\
             & 0.92 & 5.31 & 12.4 \\
             & 0.90 & 5.26 & 11.0 \\
        \bottomrule
    \end{tabular}
\end{table}

\begin{table}[htbp]
    \centering
    \small
    \caption{\textbf{Full guidance-scale sweep for \dflow-Large.} All eight settings used for the \dflow-Large frontier in Table~\ref{tab:training_guidance_pareto}, with the same evaluation protocol. $^\dagger$Non-dominated settings listed in Table~\ref{tab:training_guidance_pareto}.}
    \label{tab:l_full_sweep}
    \begin{tabular}{rrr}
        \toprule
        Guidance $\omega$ & Entropy (nats) $\uparrow$ & Gen.\ PPL $\downarrow$ \\
        \midrule
        1.5 & 5.49 & 18.0 \\
        1.6 & 5.50 & 17.7 \\
        1.7 & 5.50 & 17.4 \\
        1.8 & 5.50 & 17.4 \\
        1.9$^\dagger$ & 5.51 & 17.3 \\
        2.0$^\dagger$ & 5.51 & 17.6 \\
        2.1 & 5.50 & 17.4 \\
        2.3 & 5.48 & 17.7 \\
        \bottomrule
    \end{tabular}
\end{table}

\paragraph{Few-Step Comparison.}
The distilled ELF-Small baseline in Table~\ref{tab:elf_b_fewstep} is ELF+PD, the progressively distilled ELF-Small of \citet{elf2026}. Its reported few-step settings reach low Gen.\ PPL only at entropy $5.04$--$5.16$, so we compare at matched entropy instead. For each model and NFE, we sweep the solver and guidance scale (the sampling temperature for AR-Small) and report the lowest Gen.\ PPL among settings with entropy at least $5.40$, using 1,024 sequences per setting. FMLM~\citep{lee2026flow} reports a single setting per NFE, so we list its reported OpenWebText results and entropy directly. The \dflow settings lie at entropy $5.41$--$5.45$. Table~\ref{tab:elf_b_fewstep} lists all values.

\begin{table}[htbp]
    \centering
    \small
    \caption{\textbf{Few-step generation at matched entropy.} \dflow rows use two reflow rounds. Except for FMLM, each entry is the lowest Gen.\ PPL found with entropy at least $5.40$. FMLM values are as reported by \citet{lee2026flow}.}
    \label{tab:elf_b_fewstep}
    \begin{tabular}{lrrr}
        \toprule
        Model & NFE & Entropy (nats) $\uparrow$ & Gen.\ PPL $\downarrow$ \\
        \midrule
        \dflow-Small & 4  & 5.41 & 116.4 \\
                    & 8  & 5.42 & 51.2 \\
                    & 16 & 5.41 & 26.4 \\
                    & 32 & 5.43 & 22.9 \\
        \midrule
        \dflow-Medium & 4  & 5.42 & 113.7 \\
                    & 8  & 5.42 & 42.7 \\
                    & 16 & 5.42 & 22.2 \\
                    & 32 & 5.45 & 18.4 \\
        \midrule
        \dflow-Large & 4  & 5.43 & 117.5 \\
                    & 8  & 5.43 & 40.7 \\
                    & 16 & 5.41 & 19.4 \\
                    & 32 & 5.43 & 16.2 \\
        \midrule
        ELF+PD & 4  & 5.40 & 68.7 \\
                        & 8  & 5.42 & 66.4 \\
                        & 16 & 5.41 & 67.7 \\
                        & 32 & 5.41 & 66.3 \\
        \midrule
        FMLM & 4  & 5.26 & 111.3 \\
             & 8  & 5.36 & 86.5 \\
             & 16 & 5.29 & 63.6 \\
             & 32 & 5.25 & 45.1 \\
        \midrule
        AR-Small & 1,023 & 5.43 & 27.7 \\
        \bottomrule
    \end{tabular}
\end{table}

\subsection{Additional Generation Results}
\label{app:additional_results}

\paragraph{Quality--Diversity Trade-offs by Model Size.}
Figure~\ref{fig:pareto_by_size} separates the comparison in Figure~\ref{fig:system_comparison}(a) by model size, showing the guidance sweeps alongside AR temperature sweeps and the reported ELF results. The evaluation protocol and tabulated values are given in Appendix~\ref{app:training_guidance_pareto}.

\begin{figure}[htbp]
    \centering
    \includegraphics[width=\textwidth]{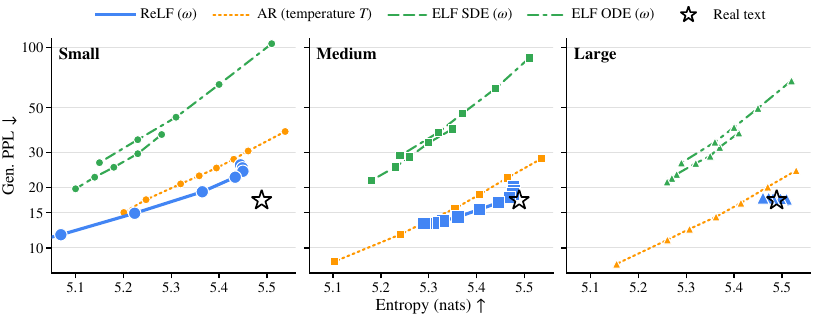}
    \caption{\textbf{Quality--diversity trade-offs by model size.} \dflow over guidance scale $\omega\in[1.5,2.5]$ and ELF~\citep{elf2026} at 64 NFE, compared with AR temperature sweeps. Higher guidance or lower temperature generally reduces entropy. \dflow uses the SDE sampler; ELF includes both SDE and ODE results as reported in the original study. Stars mark real OpenWebText. Results use 1,024 sequences per point for \dflow and AR, and 1,000 for ELF.}
    \label{fig:pareto_by_size}
\end{figure}

\paragraph{Guidance and Temperature.}
Figure~\ref{fig:radar_two_dials} extends the comparison in Figure~\ref{fig:system_comparison}(a) to lexical and semantic diversity statistics. The guidance scale of \dflow and the sampling temperature of AR models move the semantic statistics in opposite directions: raising $\omega$ lowers the Vendi score and recall, whereas lowering $T$ raises both.

\begin{figure}[htbp]
    \centering
    \includegraphics[width=\textwidth]{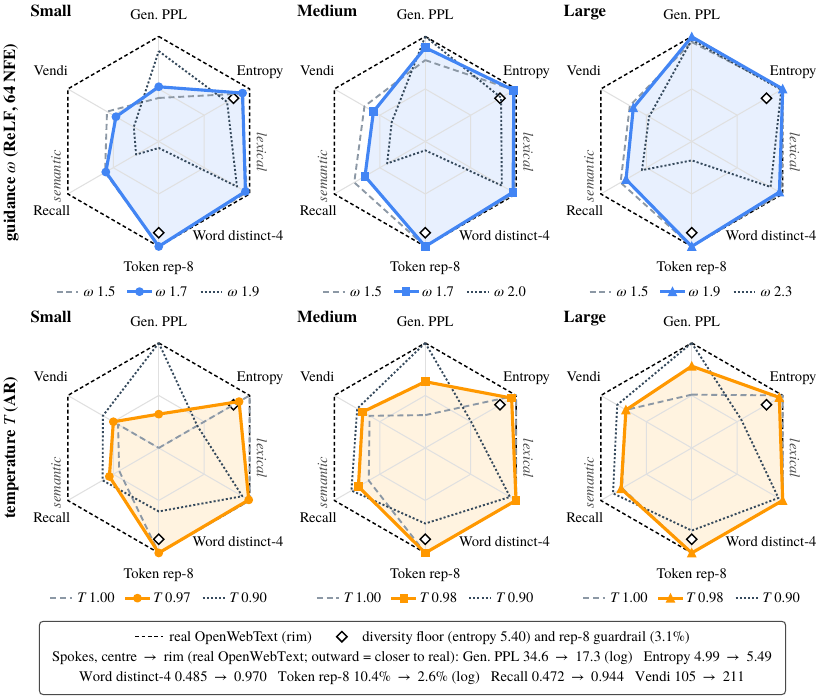}
    \caption{\textbf{Guidance and temperature across diversity statistics.} Top: \dflow at 64 NFE with the SDE sampler over guidance scale $\omega$; bottom: AR over temperature $T$. Solid polygons are the settings compared in Section~\ref{sec:headline_benchmark}; dashed and dotted polygons are a weaker and a stronger setting of the same control. On every spoke, real OpenWebText lies on the rim and outward is closer to real text; a value beyond the rim or the center is drawn on it. Vendi score and recall are computed on Qwen3-Embedding-4B embeddings of 1,024 generated documents. Distinct-4 is the fraction of distinct word 4-grams, and rep-8 is the share of a sequence's token 8-grams that recur within it. Diamonds mark the rep-8 guardrail ($3.1\%$, $1.2\times$ real text).}
    \label{fig:radar_two_dials}
\end{figure}

\paragraph{Reflow Rounds and Sampling Budget.}
Figure~\ref{fig:rounds_slope} complements the main-text sampling-budget comparison in Figure~\ref{fig:fewstep_ladders} by comparing successive reflow rounds at fixed NFE. Together, they show that the first two rounds provide most of the improvement in few-step generation. Table~\ref{tab:elf_b_fewstep} gives the comparison with ELF+PD and FMLM.

\begin{figure}[htbp]
    \centering
    \includegraphics[width=\textwidth]{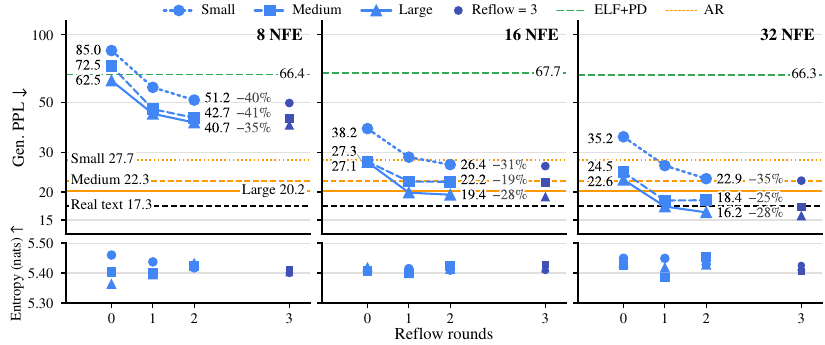}
    \caption{\textbf{Reflow rounds at fixed NFE.} Gen.\ PPL (top) and entropy (bottom) for \dflow-Small, \dflow-Medium, and \dflow-Large at 8, 16, and 32 NFE. Round 0 denotes the model before reflow; indigo markers show round 3. Annotations give Gen.\ PPL before and after two rounds and the percentage reduction. \dflow and AR results use 1,024 sequences per point. AR reference temperatures are $0.97$ for Small and $0.98$ for Medium and Large, matched to the 64-NFE \dflow entropy. Reference lines show same-size AR models, ELF+PD~\citep{elf2026} at the same NFE, and real OpenWebText.}
    \label{fig:rounds_slope}
\end{figure}

\paragraph{Latency and Throughput.}
Figure~\ref{fig:latency_throughput} expands the speedup comparison in Figure~\ref{fig:system_comparison}(b) into absolute latency and throughput. We compare generation of 1,024-token sequences at batch sizes 1, 4, 16, and 64 on one GPU, with KV caching for AR and CUDA graphs for both methods.

\begin{figure}[htbp]
    \centering
    \includegraphics[width=\textwidth]{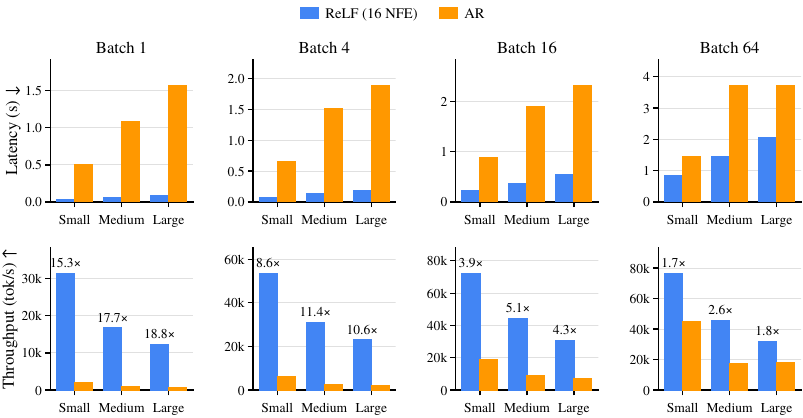}
    \caption{\textbf{Latency and throughput at matched quality.} \dflow at 16 NFE and AR on one GPU, using KV caching for AR and CUDA graphs for both methods. After two reflow rounds, \dflow achieves lower Gen.\ PPL than same-size AR models. Top: time per batch of 1,024-token sequences. Bottom: throughput in tokens per second. Labels above \dflow bars indicate speedup over AR. Batch sizes are 1, 4, 16, and 64, showing how the relative speed changes as more sequences are generated together.}
    \label{fig:latency_throughput}
\end{figure}

\clearpage
\providecommand{\eostoken}{\textless{}\textbar{}endoftext\textbar{}\textgreater{}}

\section{Qualitative Examples}
\label{app:qualitative_examples}

We show unconditional samples from \dflow-Small, \dflow-Medium, and \dflow-Large, all trained with distilled self-conditioning guidance and generated with the SDE sampler at 64 NFE. \dflow-Small and \dflow-Medium use guidance scale $\omega=1.7$, the setting reported in Section~\ref{sec:headline_benchmark}, which attains Gen.\ PPL $24.1$ at entropy $5.45$ for Small and $18.6$ at entropy $5.48$ for Medium over 1,024 samples. \dflow-Large uses $\omega=1.9$, which attains Gen.\ PPL $17.3$ at entropy $5.51$. Each sample is a sequence of 1,024 GPT-2 tokens decoded to text without any post-processing; \texttt{\eostoken} denotes the end-of-sequence token. The header of each sample reports its own Gen.\ PPL and entropy.

\begin{framed}
\noindent\textbf{\dflow-Small, sample 1}\hfill Gen.\ PPL $20.3$\quad Entropy $5.55$
\par\vspace{2pt}\hrule\vspace{6pt}
{\ttfamily\small\raggedright\setlength{\parskip}{0.6em plus 0.3em minus 0.1em}
\eostoken{} funding is a panacea. We've seen an increase in waterfalls on floodway projects over the years, which suggests that we should invest more money in areas where river infrastructure is more viable.

While New Mexico was once inundated with billions from natural rainfall, it now is heavily impacted by flooding. In recent years its reservoirs have fallen more than 100 percent. In relatively climate, there is little reason to believe that the water bills will rise more than 20 percent a year. Increased flood insurance provides a more convenient, easier-{}to-{}use way to reduce the burden of the DWF to the federal government. But it's not a panacea, either. This is what we can use right now. All elements of our infrastructure are aging, including water lines and electrical utilities. We certainly shouldn't make a repeat of new drainage infrastructure. In the absence of a reliable hydropower re drainage system, we're likely going to miss out on mudslower restoration.

The study provides evidence that providing better flood funding and reducing the risk of future flooding makes a more resilient society a realistic long-{}term alternative. In short, the best answer to the flooding crisis is by providing some public infrastructure spending\textrm{\textemdash}subsidized by both federal governments and private investors\textrm{\textemdash}to fund it in the first place.\eostoken{}$\times$ Iraqi security forces raided Christian home of wife of a Islamic State leader

KANIZALAD, Iraq (AP) \textrm{\textemdash} Iraqi security forces raided the Christian home of a family of a Islamic State leader on Sunday in an effort to curb ISIS militants in the northern Iraqi province of Nineveh, including arresting more Christians, a spokesman for the group said.

Explosvers, also known as the Shirab, raided a Christian house in the village of Hawiyah, a community where dozens of Christians lived when the raid took place.

No injuries were immediately reported. Five people were arrested, and two were wounded.

Nineveh's Yazidis have fled to fight with ISIS to the north of the Tigris, which stretches along the porous border between Iraq and Syria. It is one of the largest strongholds for Sunni Christians in the region. Yazidis have been seeking wide cover from Islamic State since it staged deadly attacks on Mosul and the city of Kirkuk in June.

Christians of the Christian faith have pledged to join forces to fight ISIS in the region, state-{}run Anadol Agency reported Wednesday.

Explosvers raided an apartment in Samarija late Sunday, and arrested at least 18 members of the Shirab, an Iraqi security official told the Associated Press.

In Hawija province, 49 members of theShirab were arrested, the official told Iraqi state news agency Dogan on Tuesday. He said at least a dozen were taken into protective custody.

Anadol news agency on Monday reported that 20 Shirab members had been allowed to leave their homes, which remained in Hawija until the night of the raid.

Hassman Kamar, a spokesman for the Christian community of Hawija, said local Christians living with theShirab could face some eviction, especially if the Iraqi government maintains its support.

"In the Nineveh, local Christians used to live with the Shirab, so they now avoid support because they are no longer allowed to leave the area," Kamar said.

Read or Share this story online: https://usat.ly/2k9zqg\eostoken{}A 28-{}year old, dual U.S. citizen has been indicted on two counts of impersonating a Dominican government official, as well as signing a false identification letter.

Rico Sino of the John Morris Community Court was charged in a two-{}count indictment. All charges threaten to defraud the U.S. government.

Originally from the Dominican Republic, Sino got his start in the subject after spending two years with Spanish Venezuela school club Monaxo.

He later asked permission to embark on his three-{}year journey into the United States. Sino's attorney Carlos Hern\'{\i}guez said to NJ Advance Media that he has never seen Sino do anything more illegal than signing a green card on the form of his \$10,000, or peso, bail.

He has been on the wrong side of the issue 'for a very long time,' he said.

State Sen. Francisco Molco said to NJ Advance Media that he still believes Sino's arrest will help the feds go after him.

'We have to see what does. I am just going to see,' he said.

The Homeland Security Department did not respond to a request for comment on the charges.

The Senate recently approved a bill that allows dual residents of the Dominican Republic to apply for citizenship in the United States.

The law covers a Dominican citizen who will not be identified by real name but has a birth card and spousal\eostoken{}\par}
\end{framed}

\begin{framed}
\noindent\textbf{\dflow-Small, sample 2}\hfill Gen.\ PPL $28.3$\quad Entropy $5.49$
\par\vspace{2pt}\hrule\vspace{6pt}
{\ttfamily\small\raggedright\setlength{\parskip}{0.6em plus 0.3em minus 0.1em}
\eostoken{} hot temperatures without the need for a 400-{}square-{}foot greenhouse that will be enough to cover 10,000 homes without electricity or water.

The Trump White House has announced a series of recovery centers to help communities in the nation's southern states overcome their dependence on fossil fuels.

So far, 1,080 units have been built so far and 895 people have been awarded permits. But high temperatures in the region are threatening key water conditions for millions of people without electricity or water.

Recovery Program Manager Bill Holley, executive director of the Union of Concerned Scientists, said the lack of adequate funding raises immediate questions, especially given that the region \textrm{\textendash} about 30 percent of the state \textrm{\textendash} suffers from extreme drought patterns.

Coupled with droughts, he said, "it's a very challenging time for temperatures to rise at high levels, especially when compared to normal temperatures. We're seeing record floods and famines, and that's something we're really concerned about," he said.

Gierra Diversity, an environmental nonprofit that helped scientists launch the Union of Concerned Scientists initiative, noted that 'as we have and do the right thing,' they will "both work together to be able to achieve more meaningful change."

Montley won't be alone. In Tucson, the group received \$25,000 from Communities From All Climate. In Yuma, it will receive \$5,000 from Community Recovery Foundation. Where Santa View is hit hardest by wildfires, it will receive \$30,950 in grants.

In Maricopa County, Gulf Water Recovery Center was awarded \$104,500 by Environmental Frontier Association for helping it meet the needs of 1,300 communities statewide. According to program coordinators for all of those counties, by the end of this year, the recovery centers will provide access to 538 people in all southern counties \textrm{\textendash} the highest number since records began in 2000.

"A recovery center can be a step toward empowering our communities," said Lana Hatterer, program manager for Make Something Clean \textrm{\textemdash}, a nonprofit that looks at how and where new communities can be built. "President Trump's announcement to build these recovery centers fulfishes that optimism and strengthens the urgency to help us address the needs we currently face."

The facility will, in a variety of ways, mirror the established facilities that will replace now. In Oahuopa and Temosa, communities without enough water will be stuck in a bitterly water water for decades to come. In Sonima, communities without sufficient water will face water cuts and many rural communities will be trapped in the cold water for decades to come.

The communities in site will be divided into 10 different manufacturers. The unit will remain operational for years. But Corley said as the facilities become more sophisticated there's growing interest in making how and supplement them for their communities.

Each unit has two basic components:

An internal inverter system that generates power waves that can be controlled by an electrical outlet. A three-{}cooled lighting system can help adjust the switch.

Most of the buildings will have prices that vary from \$1,500 to \$1,050, Holley said. But with each unit, there will be an additional 110-{}square-{}foot greenhouse, the standard form factor for state-{}of-{}the-{}art climate panels.

Many of the centers will also be set up for local philanthropic organizations. Residents will receive grants from foundations affiliated with water communities in Arizona. For example, a grant from the Maricopa County Department of Public Services will be used to fund water education. A grant from the governor's mansion will go toward the Arizona Department of Agriculture.

'When people talk about this, they're talking about a recovery center,' Corley said. 'The center will support people in place to meet vital water and medical needs. It will give local communities renewed hope for their future."

As the oil market continues and demand runs high, so do the numbers of Americans who overdose from prescription opioids.

The National Institute on Drug Abuse estimated that overdoseizations associated with the opioid overdose crisis will increase from 1,080 to 841,080 by 2020, according to a study in Baltimore, Massachusetts. But opioid overdoses in the United States, including heroin overdoses, have climbed sharply and more than 10,300 Americans have died from overdoses since 2014 alone.

DAURA WCCENBERY

NeiComFounder

Carwater Action Foundation\eostoken{}A top official on President Trump's "America First" campaign team had his license revoked after calling him a "bigoted homophobe."

Curtel Zaklow had finished his legal degree from Trump University, but a federal district judge ruled the Chicago graduate had his license revoked when he sued the University of\eostoken{}\par}
\end{framed}

\begin{framed}
\noindent\textbf{\dflow-Medium, sample 1}\hfill Gen.\ PPL $17.0$\quad Entropy $5.33$
\par\vspace{2pt}\hrule\vspace{6pt}
{\ttfamily\small\raggedright\setlength{\parskip}{0.6em plus 0.3em minus 0.1em}
\eostoken{} cannabis research firm. The company was founded in Colorado by entrepreneur David Conler in 2014. Whitler is the Principal Medical Cannabis Advisor for the Colorado Department of Health and Community Services.

You are known for your work in cannabis research, innovation, and marketing. Can you talk a bit about your goals when it comes to cannabis and the firm's strategies to achieve those goals?

Whyte Group uses a multidisciplinary, research-{}driven business model and seeks to create solutions that deviate from three pillars of cannabis, with primary objectives: reducing harmful use and underage advertising, supporting the commercialization of FDA-{}approved products, and breaking down non-{}trivial barriers to entry. The Group understands the importance of the legal cannabis industry and believes the key for achieving success is "the science."

Who are you looking to hire?

In addition to hiring researchers and innovators, the Whyte Group employs lawyers and legal people to advise the unique issues within the cannabis industry. Our leadership team also has extensive experience with education, marketing, and pharmaceutical and fid Taskancy management.

How has the research been conducted with lawmakers across the country? How has it affected both your clients and the politicians themselves?

The Whyte Group has held cannabis policy discussions with legislators in 30 states since the opening of our office in January. These discussions revolve addressed a number of policy-{}related topics such as research, development, regulation, and regulation. Our Group also works with state and local marijuana leaders, helping to shape the discussion about the issue and help policy makers to find solutions to their problems. We hope these discussions broaden the dialogue between policymakers and experts on cannabis policy, and increase legislators' support for raising the legal age of all cannabis products to 21 years old.

When you're a scientist, what is it like to get an epiphany? What was an exciting epiphany you went through in your lab when you made groundbreaking discoveries? When you're an entrepreneur, what's it like to have an exciting epiphany?

The regulatory conversation around cannabis across the country is rapidly changing overnight. While the drug has now been legalized in 23 states, the cannabis industry remains much more active in regulation. Large scale use and distribution of marijuana can still be profited under various criminal agencies, and products that aren't sold as legitimate continue to be prohibited. The Whyte Group will now work hard to ensure that these problems evolve and the industry focuses on business models that will benefit the general consumer and create good-{}paying jobs. We've successfully responded to this shift, and continue to study the issue informally.

You're a small company, with 15 people working on research, development, marketing, sales, product management. What are the firm's goals when it comes to cannabis?

When you're an entrepreneur, cannabis is often one of the first priorities that you think of things that matter to you. That said, almost two years ago, we decided to form the Whyte Group, and we believe we have a significant job ahead of us moving forward. The Whyte Group's first goal is to create a world-{}class cannabis research laboratory. Our second is to utilize a multidisciplinary, research-{}driven business model. We will engage industry stakeholders by creating business models and products that deviate from these three pillars. Creating a business model that is responsible, ethical, and minimally harmful is our third goal, and we seek to ensure cannabis is regulated and regulated the same way alcohol and tobacco. With these "3 pillars" in place, the Group will continue to engage people across the country to help the cannabis industry move forward.

The Whyte Group also believes the federal government should stop regulating the cannabis industry like it did in the tobacco industry. We believe that the government still holds sway over the regulatory process for cannabis, but is finally making some significant progress. The DEA is currently developing a list of pharmaceutical products that will be put on "Schedule A." CBD products like the Group's THC Caze Solution are also slated to be placed on Schedule B, bringing it one step closer to a Schedule A substance. In addition, the Drug Enforcement Administration is currently revising its Schedule Enforcement Controlled Drug list, which is below Schedule A. These policy changes are intended to try and take the stigma off cannabis now that federal regulation of cannabis is over.

Can you give your thoughts on the current situation in the cannabis industry?

We're cautiously optimistic, and that's for one big reason: With a market capitalization of over \$1 trillion, the cannabis industry is the fastest growing industry in the country. It's poised to be not only the gateway to some of the top jobs in this country, but also an innovation engine. Many of those resources are still badly underutilized. We'll continue our efforts to create new, exciting cannabis\eostoken{}\par}
\end{framed}

\begin{framed}
\noindent\textbf{\dflow-Large, sample 1}\hfill Gen.\ PPL $19.2$\quad Entropy $5.51$
\par\vspace{2pt}\hrule\vspace{6pt}
{\ttfamily\small\raggedright\setlength{\parskip}{0.6em plus 0.3em minus 0.1em}
\eostoken{} rural Internet service. This amendment was signed into law with the HEREAM Act in 2009. Arartz is an IEEE Fellow Emeritus of the American Institute of Electrical and Electronics Engineers.

Ginson is a researcher at North Carolina State University, where he's led a number of projects in lifesaving medicine, including developing the world's first portable dialhetic device.

Richard Kamer, a provinticisticist and former University College resident, is an IEEE Fellow Emeritus of the American Institute of Electrical and Electronics Engineers. His research has focused on constructing wired Bypass Wirelessways that provide real-{}speed access to the Internet at non-{}trivial cost to consumers. Levkin is the author of the acclaimed SmartCities, a book outlining building blocks for "smart cities."

What new graduates need to know?

In addition to offering scholarships and mentoring, the Letch Institute hosts conferences and workshops designed to help new graduates build upon their advanced degrees. Our flagship conference is held annually and attracts 1,000+ students from over 100 companies and countries. Our conference were named the 2014 World School Lecture of the Year and Young Leaders Worksation of the Year by Young Scouts America.

The Letch Institute also offers an online developer Development program that prepares students with the skills to become successful Web developers. Our classes are offered twice a year to web-{}agers ages 13 to 25, regardless of high location or income. Each day, students learn the skills and techniques that are required to use the Internet as a resource and create useful applications that enable others to access it. In addition to teaching students these tools and skills, our online training program also offers free Master's and Doctorate Degree programs for students ages 18 to 35 years old.

When I've asked new graduates what it feels like to be a Letch Institute Fellow, the overwhelming response is, "I've just earned your degree, and you're still getting paid!"

What is the Rural Internet Innovation Fund?

The Rural Internet Innovation Fund is a federal program to help rural communities build more reliable Internet infrastructure so everyone in any parts of the country can get Internet for anywhere at home.

More than half of rural Americans already have high-{}speed Internet service connections, and they will be the most widely used if broadband continue to be expanding. The RIEF can promote rural growth by building more reliable Internet infrastructure in rural communities and by funding grants that will help those communities grow and have good-{}paying jobs. That's less waiting for cable TV and more opportunities to pitch your sweet cash pump.

Here's a great YouTube video about this initiative released on Jan. 21, 2014, and for much more information, see the President's FY14 State of the Union Address.

You're the president, CEO, and CEO of American River Code. What do those three positions mean to you?

After working at Silicon University, I decided to launch the Letch Institute, and since then I have been responsible for all of our leadership activities. The Letch Institute's core mission is to be the driving driving force for digital innovation. Our focus is to keep the innovation world open to all backgrounds and all perspectives. While we value intellectual property rights, we believe that human property everywhere belong in the digital realm. Creating a space where we can collaborate, connect, connect, and create is a huge gift, and we have to make sure those currently excluded from the digital world can experience it. Since the "Arab Spring" in 2010, digital rights movements have empowered ordinary people around the world to discuss the critical issues we face.

The Letch Institute also believes that American businesses should contribute to the digital economy like they contribute to the real economy. Contributing to the economy means supporting them by maximizing economic returns and creating wealth for consumers. For more information about our global mission, please visit the Letch Institute website: http://www.letchinstitute.org/.

Part of our nation's economic success is putting people to work to meet our true economy. Access to the Internet is key to economic growth and prosperity. In 2000, only 1 percent of the country had an Internet connection, but by 2012, that increased to 18 percent. We need our private courage to invest to make the Internet as accessible and affordable as possible for generations to come.

Have you done any work in the private sector or the public sector?

I'm an engineer, and I've been programming since 1985. My private sector career began at Ruralron (now Southern Grid, Inc.), a company that specializes in Web storage. I'm currently chief technology officer at Public Electric, which is the largest utility in the state, serving over 1 million customers. Many of those are economically underserved in rural Texas and don't have Internet access. My public sector career began\eostoken{}\par}
\end{framed}

\end{document}